\documentclass{article}

\usepackage{algorithm}
\usepackage{algorithmic}

\usepackage[margin=1.5in]{geometry}

\usepackage[utf8]{inputenc} 
\usepackage[T1]{fontenc}    
\usepackage{hyperref}       
\usepackage{url}            
\usepackage{booktabs}       
\usepackage{amsmath}
\usepackage{amsfonts}
\usepackage{dsfont}
\usepackage{amssymb}
\usepackage{amsthm}
\usepackage{nicefrac}       
\usepackage{microtype}      
\usepackage{xcolor}         
\usepackage{authblk}
\usepackage{graphicx}
\newtheorem{theorem}{Theorem}
\newtheorem{corollary}{Corollary}
\newtheorem{lemma}{Lemma}
\newtheorem*{theorem*}{Theorem}
\newtheorem*{corollary*}{Corollary}
\newtheorem*{lemma*}{Lemma}

\newtheorem{proposition}{Proposition}
\newtheorem{definition}{Definition}
\newtheorem{property}{Property}
\newtheorem*{property*}{Property}

\title{Efficient Policy Iteration for Robust Markov Decision Processes via Regularization}

\author[1]{Navdeep Kumar}
\author[1]{Kfir Levy}
\author[2]{Kaixin Wang}
\author[1]{Shie Mannor}

\affil[1]{Technion}
\affil[2]{National University of Singapore}

\begin{document}
\maketitle
\begin{abstract}
Robust Markov decision processes (MDPs) provide a general framework to model decision problems where the system dynamics are changing or only partially known. Efficient methods for some \texttt{sa}-rectangular robust MDPs exist, using its equivalence with reward regularized MDPs, generalizable to online settings. In comparison to \texttt{sa}-rectangular robust MDPs,  \texttt{s}-rectangular robust MDPs are less restrictive but much more difficult to deal with.  Interestingly, recent works have established the equivalence between \texttt{s}-rectangular robust MDPs and policy regularized MDPs. But we don't have a clear understanding to exploit this equivalence, to do policy improvement steps to get the optimal value function or policy. We don't have a clear understanding of greedy/optimal policy except it can be stochastic. There exist no methods that can naturally be generalized to model-free settings. We show a clear and explicit equivalence between \texttt{s}-rectangular $L_p$ robust MDPs and policy regularized MDPs that resemble very much policy entropy regularized MDPs widely used in practice.  Further, we dig into the policy improvement step and concretely derive optimal robust Bellman operators for \texttt{s}-rectangular $L_p$ robust MDPs.  We find that the greedy/optimal policies in \texttt{s}-rectangular $L_p$ robust MDPs are threshold policies that play top $k$ actions whose $Q$ value is greater than some threshold (value), proportional to the $(p-1)$th power of its advantage.  In addition, we show time complexity of  (\texttt{sa} and \texttt{s}-rectangular) $L_p$ robust MDPs is the same as non-robust MDPs up to some log factors. Our work greatly extends the existing understanding of \texttt{s}-rectangular robust MDPs and naturally generalizable to online settings.
\end{abstract}

\section{Introduction}
In Markov Decision Processes (MDPs), an agent interacts with the environment and learns to optimally behave in it ~\cite{Sutton1998}.
Nevertheless, an MDP solution may be very sensitive to the model parameters ~\cite{BiasVarianceShie,genralization1,generalization2}, implying that one should be cautious when the model is changing or when there is uncertainty in the model parameters. Robust MDPs (RMDPs) on the other hand, allow some room for the uncertainty in the model parameters as it allows the model parameters to belong to some uncertainty set ~\cite{HKuhn2013,tamar14,Iyenger2005}. 

Solving robust MDPs is NP-hard for general uncertainty sets \cite{Nilim2005RobustCO}, hence the uncertainty set is popularly assumed to be \textit{rectangular} which enables tractability \cite{Iyenger2005,Nilim2005RobustCO,ppi}.
The uncertainty set is called rectangular if the uncertainty in model parameters (\textit{i.e.}, reward and transition kernel) in one state is not coupled with the uncertainty in model parameters in different state.
Under this rectangularity assumption, many structural properties of MDPs remain intact\cite{Iyenger2005, Nilim2005RobustCO} and methods such as robust value iteration \cite{Bagnell01solvinguncertain,RVI}, robust modified policy iteration \cite{Kaufman2013RobustMP}, partial robust policy iteration \cite{ppi} etc can be used to solve it.
It is also known that uncertainty in the reward can be easily handled, whereas handling uncertainty in transition kernel is much more challenging \cite{derman2021twice}.
When an uncertainty set has a polyhedral structure (as in $L_1/L_\infty$ robust MDPs) then it can be solved using nested Linear programming, that is, both policy evaluation and policy improvement are done using linear programming exactly or inexactly\cite{Kaufman2013RobustMP,RVI}. 
Further in this direction, the partial policy iteration approach to solve $L_1$-Robust MDPs by \cite{ppi} uses bags of tricks such as homotopy, bisection methods etc.
But these methods are conceptually complex, offering very little insight, and unable to generalize to online settings.  Work by \cite{Rcontamination} considers the R-contamination uncertainty set and proposes a Q-learning algorithm that can be used in online settings too. But their uncertainty set is \texttt{sa}-rectangular, interestingly their results resembles very closely to our \texttt{sa}-rectangular $L_1$ robust MDPs results. We believe, our approach can be used to generalize R-contamination uncertainty set to \texttt{s}-rectangular settings as well. Goyal et al. \cite{r-rectRMDP} consider an r-rectangular uncertainty set that is not s-rectangular, still, the setting is simple enough to admit deterministic policies as optimal policy. A fundamental differences between \texttt{sa} and \textbf{s} rectangular robust MDPs is the greedy and optimal policy in \texttt{sa}-rectangular robust MDPs like non-robust MDPs, is deterministic, but it may be stochastic in \texttt{s}-rectangular  case \cite{RVI}.


Several previous works have investigated robust MDPs from a regularization perspective \cite{derman2021twice,DRO_derman,rewardRobustRegularization,MaxEntRL}.
While this approach provides an interesting insight into the problem at hand, it has so far did not lead to scalable methods for solving RMDPs. 
Along this direction and closest to our work is by Derman et. al \cite{derman2021twice} which showed that robust MDPs can be viewed as value and policy regularized MDPs.  
Particularly, for $(\mathtt{sa})$-rectangular case, they reduced robust policy iteration (both policy evaluation and improvement) to reward regularized non-robust MDPs. 
The reward regularizer comes from the uncertainty in the reward and transition kernel. Nevertheless, their assumption of kernel noise is unrealistic.
For $\mathtt{s}$-rectangular case, they obtained only policy evaluation as regularized non-robust MDPs, and for policy improvement, they go for an external black box solver (projection onto simplex solver). Hence, making it very difficult to generalize in an online setting. 
 

\textbf{Contributions:}
\begin{itemize}
    \item Established a concrete equivalence between \texttt{s}-rectangular robust MDPs to policy regularized MDPs.
    \item A clear understanding of the policy improvement step leads to efficient computation of optimal value and policy that is naturally generalizable to an online setting.
    \item A clear characterization of the greedy/optimal policy in \texttt{s}-rectangular $L_p$ robust MDPs. To the best of our knowledge, this kind of policy is novel in literature. Like the softmax policy, these policies give more weight to actions having more advantages but unlike softmax, it doesn't play utterly useless actions.
    \item We provide a complexity analysis of our methods for  $L_p$ robust MDPs that are the same as non-robust MDPs up to some log factors.
    \item We correct unrealistic assumption made in noise transition kernel in \cite{derman2021twice} in both \texttt{sa} and \texttt{s} rectangular case.
    
    
\end{itemize}  
It is noteworthy that our work can be used to obtain an optimistic policy that is used as an exploration policy in \cite{UCRL, UCRL2}. Interestingly, we show UCRL 2 algorithm (in discounted setting) resembles UCB VI \cite{UCRLVI}. This relation is discussed in section \ref{app:UCRL} in details.

\section{Preliminary}
\subsection{Notations}
For a set $\mathcal{S}$, $\lvert\mathcal{S}\rvert$ denotes its cardinality. $\langle u, v\rangle := \sum_{s\in\mathcal{S}}u(s)v(s)$ denotes the dot product between functions $u,v:\mathcal{S}\to\mathbb{R}$. $\lVert v\rVert_p^q :=(\sum_{s}\lvert v(s)\rvert^p)^{\frac{q}{p}}$ denotes the $q$th power of $L_p$ norm of function $v$, and we use  $\lVert v\rVert_p := \lVert v\rVert^1_p$ and $\lVert v\rVert := \lVert v\rVert_2$ as shorthand. For a set $\mathcal{C}$, $\Delta_{\mathcal{C}}:=\{a:\mathcal{C} \to \mathbb{R}|a(s)\geq 0, \forall s \sum_{c\in\mathcal{C}}a_c=1\}$ is the probability simplex over $\mathcal{C}$
$\mathbf{0},\mathbf{1} $ denotes all zero vector and all ones vector/function respectively of appropriate dimension/domain. $\mathds{1}(a=b):=1$ if $a=b$, 0 otherwise, is indicator function.

\subsection{Markov Decision Processes}\label{sec:MDPs}
A Markov Decision Process (MDP) is defined by $(\mathcal{S},\mathcal{A},P,R,\gamma,\mu)$ where $\mathcal{S}$ is the state space, $\mathcal{A}$ is the action space, $P:\mathcal{S}\times\mathcal{A} \to \Delta_{\mathcal{S}}$ is the transition kernel, $R:\mathcal{S}\times\mathcal{A} \to \mathbb{R}$ is the reward function, $\mu\in\Delta_{\mathcal{S}}$  is the initial distribution over states and $\gamma \in [0,1)$ is the discount factor \cite{Sutton1998}. A stationary policy $\pi :\mathcal{S}\to \Delta_{\mathcal{A}}$ maps states to probability distribution over actions. And $\pi(a|s), P(s'|s,a), R(s,a)$ denotes the probability of selecting action $a$ in state $s$, transition probability to state $s'$ in state $s$ under action $a$, and reward in state $\mathtt{s}$ under action $a$ respectively.  We denote 
$S=|\mathcal{S}|$, and $A=|\mathcal{A}|$ as a shorthand. The objective in a MDP is to maximize the expected discounted cumulative reward, defined as
\begin{align*}
     &\mathbb{E}\left[\sum_{n=0}^{\infty}\gamma^n R(s_n,a_n)\Bigm| s_0\sim \mu, \pi,P\right]
     = \langle R,d^\pi_{P}\rangle =\langle \mu,v^\pi_{P,R}\rangle,
\end{align*}
    
where $d^\pi_P$ is the occupation measure of policy $\pi$ with the initial distribution of states $\mu$, and kernel $P$, defined as
\[d^\pi_{P}(s,a) := \mathbb{E}\left[\sum_{n=0}^\infty \gamma^n\mathds{1}(s_n=s,a_n=a)\Bigm\vert s_0\sim \mu, \pi,P \right],\]
and $v^\pi_{P,R}$ is the value function (dual formulation\cite{Puterman1994MarkovDP}) under policy $\pi$ with transition kernel $P$ and reward vector $R$, defined as 
\begin{align}
  v^\pi_{P,R}(s) :=\mathbb{E}\left[\sum_{n=0}^\infty \gamma^nR(s_n,a_n)\Bigm\vert s_0= s,\pi,P\right].
\end{align}
The value function $v^\pi_{P,R}$ for policy $\pi$, is the fixed point of the Bellmen operator $\mathcal{T}^\pi_{P,R}$  \cite{Sutton1998}, defined as 
\begin{align*}
    (\mathcal{T}^\pi_{P,R} v)(s) = \sum_{a}\pi(a|s)[R(s,a) + \gamma \sum_{s'}P(s'|s,a)v(s')].
\end{align*}
Similarly, the optimal value function $v^*_{P,R}:= \max_{\pi}v^\pi_{P,R}$ is well defined and is the fixed point of the optimal Bellman operator $\mathcal{T}^*_{P,R}$ \cite{Sutton1998}, defined as
\begin{align*}
    (\mathcal{T}^*_{P,R} v)(s) = \max_{a\in\mathcal{A}}\bigm[R(s,a) + \gamma \sum_{s'}P(s'|s,a)v(s')\bigm].
\end{align*}
The optimal Bellman operator $\mathcal{T}^*_{P,R}$ and  Bellman operators $\mathcal{T}^\pi_{P,R}$  for all policy $\pi$, are $\gamma$-contraction  maps\cite{Sutton1998}. 
So
the sequences $\{v^\pi_n|n\geq 0\}$, and  $\{v^*_n|n\geq 0\}$ defined as
\begin{align}
    v^\pi_{n+1} := \mathcal{T}^\pi_{P,R}v^\pi_n,\qquad v^*_{n+1} := \mathcal{T}^*_{P,R}v^*_n
\end{align}
converge linearly to $v^\pi_{P,R}$ the value function for policy $\pi$, and $v^*_{P,R}$ the optimal value function respectively for all initial value $v^\pi_0$ and $v^*_0$.

\subsection{Robust Markov Decision Processes}
In most practical cases, the system dynamics (transition kernel $P$ and reward function $R$) are not known precisely. Instead, we have an access to a nominal reward function $R_0$ and a nominal transition kernel $P_0$ that may have some uncertainty. To capture this, the uncertainty (or ambiguity) set is defined as $\mathcal{U}:= (R_0 + \mathcal{R})\times(P_0 + \mathcal{P})$,  where $\mathcal{R}$, $\mathcal{P}$ are the respective uncertainties  in the reward function and transition kernel \cite{Iyenger2005}.The robust performance of a policy $\pi$ is defined to be its worst performance on the entire uncertainty set $\mathcal{U}$. And our objective is to maximize the robust performance, that is 
\begin{align}
    \max_{\pi}\min_{{R,P\in\mathcal{U}}} \langle R,d^\pi_{P}\rangle,\quad \text{or}\quad
   \max_{\pi}\min_{{R,P\in\mathcal{U}}} \langle \mu,v^\pi_{P,R}\rangle.
\end{align}
Solving the above robust objective is NP-hard in general\cite{Iyenger2005,RVI}. 
Hence, the uncertainty set $\mathcal{U}$ is commonly assumed to be $\mathtt{s}$-rectangular, that is $\mathcal{R}$ and $\mathcal{P}$ can be decomposed state wise as $\mathcal{R} =\times_{s\in\mathcal{S}}\mathcal{R}_s$ and $\mathcal{P}=\times_{s\in\mathcal{S}}\mathcal{P}_s$. Sometimes, the uncertainty set $\mathcal{U}$ is assumed to decompose state-action wise as $\mathcal{R} =\times_ {s\in\mathcal{S},a\in\mathcal{A}}\mathcal{R}_{s,a}$ and $\mathcal{P}=\times_{s\in\mathcal{S},a\in\mathcal{A}}\mathcal{P}_{s,a}$, known as $(\mathtt{sa})$-rectangular uncertainty set. Observe that it is a special case of $\mathtt{s}$-rectangular uncertainty set. Under the $\mathtt{s}$-rectangularity assumption,  many structural properties of MDPs stay intact, and the problem becomes tractable \cite{RVI,Iyenger2005,Nilim2005RobustCO}.Throughout the paper, we assume that the uncertainty set $\mathcal{U}$ is $\mathtt{s}$-rectangular (or $(\mathtt{sa})$-rectangular) unless stated otherwise. Under $\mathtt{s}$-rectangularity,  the robust value function is well defined \cite{Nilim2005RobustCO,RVI,Iyenger2005} as 
\begin{align}
    v^\pi_{\mathcal{U}} := \min_{{R,P\in\mathcal{U}}}v^\pi_{P,R}.
\end{align}
Using the robust value function, robust policy  performance can be rewritten as 
\begin{align}\begin{aligned}
    \min_{{R,P\in\mathcal{U}}} \langle \mu,v^\pi_{P,R}\rangle &= \langle \mu,\min_{{R,P\in\mathcal{U}}}v^\pi_{P,R}\rangle=  \langle \mu,v^\pi_{\mathcal{U}}\rangle.
\end{aligned}\end{align}
The robust value function $v^\pi_{\mathcal{U}}$ is the fixed point of the robust Bellmen operator $\mathcal{T}^\pi_{\mathcal{U}}$  \cite{RVI,Iyenger2005}, defined as 
\begin{align*}
    &(\mathcal{T}^\pi_{\mathcal{U}} v)(s) := \min_{{R,P\in\mathcal{U}}}\sum_{a}\pi(a|s)\left[R(s,a) + \gamma \sum_{s'}P(s'|s,a)v(s')\right].
\end{align*}
Moreover, the optimal robust value function $v^*_{\mathcal{U}}:= \max_{\pi} v^\pi_{\mathcal{U}}$ is well defined and is the fixed point of the optimal robust Bellman operator $\mathcal{T}^*_{\mathcal{U}}$ \cite{Iyenger2005,RVI}, defined as
\begin{align*}
    (\mathcal{T}^*_{\mathcal{U}} v)(s) := &\max_{\pi_s\in\Delta_\mathcal{A}}\min_{{R,P\in\mathcal{U}}}\sum_{a}\pi_s(a)\bigm[R(s,a) + \gamma \sum_{s'}P(s'|s,a)v(s')\bigm].
\end{align*}
The optimal robust Bellman operator $\mathcal{T}^*_{\mathcal{U}}$ and  robust Bellman operators $\mathcal{T}^\pi_{\mathcal{U}}$ are $\gamma$ contraction maps for all policy $\pi$ (see theorem 3.2 of \cite{Iyenger2005}), that is 
\begin{align*}
    &\lVert\mathcal{T}^*_{\mathcal{U}}v - \mathcal{T}^*_{\mathcal{U}}u\rVert_{\infty} \leq \gamma \lVert u-v\rVert_{\infty},\\ \qquad & \lVert\mathcal{T}^\pi_{\mathcal{U}}v - \mathcal{T}^\pi_{\mathcal{U}}u\rVert_{\infty} \leq \gamma \lVert u-v\rVert_{\infty},\qquad \forall \pi,u,v.
\end{align*}
So for all initial values $v^\pi_0,v^*_0$, sequences defined as 
\begin{align}
    v^\pi_{n+1} := \mathcal{T}^\pi_{\mathcal{U}}v^\pi_n , \qquad v^*_{n+1} := \mathcal{T}^*_{\mathcal{U}}v^*_n
\end{align}
converges linearly to their respective fixed points, that is $v^\pi_n\to v^\pi_{\mathcal{U}}$ and $v^*_n\to v^*_{\mathcal{U}}$. This makes the robust value iteration an attractive method for solving robust MDPs.

\section{Method}
In this section, we study \texttt{s} and \texttt{sa}-rectangular $L_p$ robust MDPs where the uncertainty set is constrained by $L_p$ balls, naturally, arises in the many situations \cite{derman2021twice,ppi,UCRL2}. We will see in theorem \ref{rs:saLprvi} that for the (\texttt{sa})-rectangular case,  both policy evaluation and improvement can be done similarly to non-robust MDPs with an additional reward penalty. Theorem \ref{rs:SLpPlanning} states that policy evaluation for the \texttt{s}-rectangular case, is similar to (\texttt{sa})-rectangular case except the reward penalty has extra dependence on policy. This dependence on the policy makes policy improvement a challenging task, and the rest of the section is devoted to efficiently solving it. Theorem \ref{rs:rve} provides methods for the policy improvement for \texttt{s}-rectangular case, using binary search for general $p$. Although there exists an algorithmic method that is tailor-made for $p=1,2$, it is discussed in the appendix in detail. However, the best of both worlds is summarized in table \ref{tb:val}.  Theorem \ref{rs:srect:optimalPolicy} characterizes the stochasticity in the optimal policy. All the results are summarized in the table \ref{tb:val} and \ref{tb:Opt:pi}. There are many additional properties, and characterizations of greedy policy are discussed in appendix \ref{app:properties}.\\

We begin by making a few useful definitions. Let $q$ be such that it satisfies the Holder's equality, i.e. $\frac{1}{p} + \frac{1}{q} = 1$. 
Let $p$-variance function  $\kappa_p:\mathcal{S}\to\mathbb{R}$ and $p$-mean function  $\omega_p:\mathcal{S}\to\mathbb{R}$  be defined as
\begin{equation}\label{def:kp}
    \kappa_p(v) :=\min_{\omega\in\mathbb{R}}\lVert v-\omega\mathbf{1}\lVert_p, \quad \omega_p(v) := arg\min_{\omega\in\mathbb{R}}\lVert v-\omega\mathbf{1}\lVert_p. 
\end{equation}
 $\omega_p(v)$ can be calculated by binary search in the range $[\min_{s}v(s),\max_{s}v(s)]$ and can then be used to approximate $\kappa_p(v)$ (see appendix \ref{app:pvarianceSection}). Observe that for $p=1,2,\infty$, the $p$-variance function $\kappa_p$ can also be computed in closed form, see table \ref{tb:kappa} for summary and proofs are in appendix \ref{app:pvarianceSection}. And proofs of all upcoming results in this section can be found in the appendix \ref{app:srLp}.

\begin{table}[H]
  \caption{$p$-variance}
  \label{tb:kappa}
  \centering
  \begin{tabular}{lll}
    \toprule                   
    $x$     & $\kappa_x(v)$     & remark \\
    \midrule
    $p$ & $\min_{\omega\in\mathbb{R}}\lVert v-\omega\mathbf{1}\lVert_p  $  &  binary search     \\&\\
    $\infty$     & $\frac{1}{2}\bigm(\max_{s}v(s) - \min_{s}v(s)\bigm)$ & peak to peak difference      \\&\\
    $2$     & $\sqrt{\sum_{s}\bigm(v(s) -\frac{\sum_{s}v(s)}{S}\bigm)^2}$      & Variance  \\&\\
    $1$     &$ \sum_{i=1}^{\lfloor (S+1)/2\rfloor}v(s_i)  -\sum_{i =\lceil (S+1)/2\rceil}^{S}v(s_i))$  & Top half - bottom half    \\
    \bottomrule
  \end{tabular}
  \begin{tabular}{l}
      where $v$ is sorted, i.e. $v(s_i)\geq v(s_{i+1}) \quad \forall i.$
  \end{tabular}      
\end{table}

\subsection{\texttt{(Sa)}-rectangular $L_p$ robust Markov Decision Processes}
In accordance with \cite{derman2021twice}, we define $(\mathtt{sa})$-rectangular $L_p$ constrained uncertainty set as  
\[\mathcal{U}^{\mathtt{sa}}_p := (R_0 +  \mathcal{R})\times(P_0 +  \mathcal{P})\quad \text{where}\quad\mathcal{P} = \times_{s\in\mathcal{S},a\in\mathcal{A}}\mathcal{P}_{s,a} \]
\[ \mathcal{P}_{s,a} = \{p_{s,a}:\mathcal{S}\to\mathbb{R}\mid \underbrace{\sum_{s'}p_{s,a}(s')=0}_{\text{ condition A}},  \lVert p_{s,a}\rVert_p\leq \beta_{s,a}\},\]
\[  \mathcal{R} = \times_{s\in\mathcal{S},a\in\mathcal{A}}\mathcal{R}_{s,a},\mathcal{R}_{s,a} =  \{r_{s,a}\in\mathbb{R}\mid \lVert r_{s,a}\rVert_p\leq \alpha_{s,a}\}\]

and $\alpha_{s,a}, \beta_{s,a}\in\mathbb{R}$ are reward noise radius and transition kernel noise radius respectively. These are chosen small enough so that all the transition kernels in  $(P_0 +\mathcal{P})$ are well defined.  Since the sum of rows of valid transition kernel must be one, hence the sum of rows of noise kernel must be zero (condition A), which we ensured in the above definition of $\mathcal{P}_{s,a}.$ It makes our result differ from \cite{derman2021twice}, as they did not impose this condition (condition A) on their kernel noise. This renders their setting unrealistic (see corollary 4.1 in \cite{derman2021twice}) as not all transition kernels in their uncertainty set, satisfy the properties of transition kernels. And we see next and later, this condition on the noise makes reward regularizer dependent on  the $\kappa_q(v)$ ( $q$-variance of value function) in our result instead of $\lVert v\rVert_q$ ($q$th norm of value function $v$) in \cite{derman2021twice}. Now, we focus on the cases where $P_0(s'|s, a) =0$, that is, forbidden transitions. In most practical situations, for a given state, many transitions are impossible. For example, consider a grid world example where only a single-step jump (left, right, up, down) is allowed, so in this case, the probability of making a multi-step jump is impossible. Formally, let $\mathcal{F}_s$ be the set of forbidden states in state $s$, and  nominal kernel $P_0$ and all noise kernel in $\mathcal{P}$  must satisfy
\begin{align*}
    P_0(s'|s,a) = P(s'|s,a) = 0  , \quad \forall a \in \mathcal{A}, P \in \mathcal{P}, s'\in \mathcal{F}_s. 
\end{align*}
As it makes no sense, upon adding noise to the kernel, the system starts making impossible transitions. These constrained can be naturally incorporated without much change in theory below, hence we discuss it in appendix \ref{app:ForbiddenTransition}. This is also one of our novel contribution.\\

 \begin{theorem}\label{rs:saLprvi} $(\mathtt{Sa})$-rectangular $L_p$ robust Bellman operator is equivalent to reward regularized (non-robust) Bellman operator. That is, we have 
\begin{align*}
    (\mathcal{T}^\pi_{\mathcal{U}^{\mathtt{sa}}_p} v)(s)  =& \sum_{a}\pi(a|s)\Bigm[  -\alpha_{s,a} -\gamma\beta_{s,a}\kappa_q(v)  +R_0(s,a) +\gamma \sum_{s'}P_0(s'|s,a)v(s')\Bigm],\\
    (\mathcal{T}^*_{\mathcal{U}^{\mathtt{sa}}_p} v)(s)  =& \max_{a\in\mathcal{A}}\Bigm[  -\alpha_{s,a} -\gamma\beta_{s,a}\kappa_q(v)  +R_0(s,a) +\gamma \sum_{s'}P_0(s'|s,a)v(s')\Bigm].
\end{align*}
\end{theorem}
\begin{proof}
The proof mainly consists of two parts: a) Separating the noise from nominal values, along the lines of \cite{derman2021twice}. b) The reward noise to yields the term $-\alpha_{s,a}$ and noise in kernel yields $ -\gamma\beta_{s,a}\kappa_q(v)$. 
\end{proof}
The above result states that robust value iteration (both policy evaluation and improvement ) can be done as easily as non-robust value iteration with reward penalty. Notably, the reward penalty is not only proportional to the uncertainty radiuses but also to the $\kappa_p(v)$ that measures the variance in value function. We retrieve non-robust value iteration  by putting uncertainty radiuses (i.e. $\alpha_{s,a},\beta_{s,a}$) to zero, in the above results. Furthermore, the same is true for all subsequent robust results in this paper. The above result implies the Q-learning of the following form 
\begin{align*}
    Q_{n+1}(s,a)& = \max_{a}\Bigm[R_0(s,a)-\alpha_{s,a}-\gamma\beta_{s,a}\kappa_q(v_n) + \sum_{s'}P_0(s'|s,a)\max_{a}Q_n(s',a')\Bigm],
\end{align*}
where $  v_{n}(s) = \max_{a}Q_{n}(s,a)$, which  is further discussed in appendix \ref{app:SALpQL}. Observe that $p$-variance $\kappa_p(v)$  can be estimated online, using batches or other more sophisticated methods.  This paves the path for generalizing to a model-free setting similar to \cite{Rcontamination}, and policy gradient methods as obtained in \cite{PG_RContamination}.

\begin{table*}
\caption{Optimal robust Bellman operator evaluation}
  \label{tb:val}
  \centering
  \begin{tabular}{lll}
    \toprule                   
    $\mathcal{U}$     & $(\mathcal{T}^*_\mathcal{U}v)(s)$     & remark \\
    \midrule
    $\mathcal{U}^s_p$ & $ \sum_{a}\bigm( Q(s,a) - (\mathcal{T}^*_\mathcal{U}v)(s)\bigm)^{p}\mathbf{1}(Q(s,a)\geq (\mathcal{T}^*_\mathcal{U}v)(s))$  & Solve by binary search     \\&\\&=$(\sigma_q(v,s))^p$&\\&\\
    $\mathcal{U}^s_1$    & $\max_{k} \frac{1}{k}\bigm(\sum_{i}^kQ(s,a_i) -\sigma_\infty(v,s)\bigm)$  & Highest penalized average    \\&\\
    $\mathcal{U}^s_2$     &
    $f_2(Q(s,\cdot),\sigma_2(v,s))$ by algorithm $\ref{alg:main:f2}$
    & High average with high variance  \\&\\
    $\mathcal{U}^s_\infty$     &$  -\sigma_1(v,s) + \max_{a\in\mathcal{A}}Q(s,a)$     & Best response  \\&\\
    $\mathcal{U}^{\mathtt{sa}}_p$    &$  \max_{a\in\mathcal{A}}[ - \alpha_{sa} -\gamma\beta_{sa} \kappa_q(v) +Q(s,a)] $    & Best penalized response  \\
   \bottomrule
  \end{tabular}
  \begin{tabular}{l}
    where $ \quad Q(s,a) = R_0(s,a) + \gamma\sum_{s'}P_0(s'|s,a)v(s'), \qquad Q(s,a_1)\geq \cdots\geq Q(s,a_A)$,\\ and $\quad\sigma_q(v,s)= \alpha_s +\gamma\beta_{s}\kappa_q(v)$.
  \end{tabular}
\end{table*}
\subsection{\texttt{S}-rectangular $L_p$ robust Markov Decision Processes}

We define $\mathtt{s}$-rectangular $L_p$ constrained uncertainty set as  
\[\mathcal{U}^{\mathtt{s}}_p = (R_0 +  \mathcal{R})\times(P_0 +  \mathcal{P}), \quad \text{where}\quad \mathcal{P} = \times_{s\in\mathcal{S}}\mathcal{P}_{s}\]
\begin{align*}
     \mathcal{P}_{s} = \{p_s:\mathcal{S}\times\mathcal{A}&\to\mathbb{R}\Bigm| \sum_{s'}p_s(s',a)=0,\forall a, \lVert p_s\rVert_p\leq \beta_{s}\},
\end{align*}
\[\mathcal{R} = \times_{s\in\mathcal{S}}\mathcal{R}_{s}, \quad \mathcal{R}_{s} =  \{r_s:\mathcal{A}\to\mathbb{R}\Bigm| \lVert r_s\rVert_p\leq \alpha_{s}\},\]

and $\alpha_{s}, \beta_{s}\in\mathbb{R}$ are reward and transition kernel respectively noise radius that is chosen small enough so that all the transition kernels in $P_0 +\mathcal{P}$ are well defined.  \texttt{s}-rectangularity is more challenging to deal with than \texttt{sa}-rectangularity, so we begin with robust policy evaluation.

\begin{theorem} \label{rs:SLpPlanning}(Policy Evaluation) \texttt{S}-rectangular $L_p$ robust Bellman operator is equivalent to policy regularized (non-robust) Bellman operator, that is 
\begin{align*}
    (\mathcal{T}^\pi_{\mathcal{U}^s_p} &v)(s)  =   -\bigm(\alpha_s +\gamma\beta_{s}\kappa_q(v)\bigm)\lVert\pi_s\rVert_q +\sum_{a}\pi(a|s)\bigm(R_0(s,a) +\gamma \sum_{s'}P_0(s'|s,a)v(s')\bigm),
\end{align*}
where $\lVert \pi_s\rVert _q$ is $q$-norm of the vector $\pi(\cdot|s)\in\Delta_{\mathcal{A}}$.
\end{theorem}
\begin{proof}
Proof techniques are similar to as its $(\mathtt{sa})$-rectangular counterpart.
\end{proof}
Observe that $\mathtt{s}$-rectangular $L_p$ robust policy evaluation is same as its (\texttt{sa}) counterpart except here the reward penalty has an extra dependence on the norm of the policy ($\lVert \pi_s\rVert_q$). In spirit, $\lVert \pi_s\rVert_q$ is similar to entropy regularization  $\sum_{a}\pi(a|s)\ln(\pi(a|s))$ and other regularizers such as $\sum_{a}\pi(a|s)\textit{tsallis}(\frac{1-\pi(a|s)}{2})$, $\sum_{a}\pi(a|s)cos(cos(\frac{\pi(a|s) }{2})$, etc that are widely studied in the literature \cite{EntReg1,EntReg2,EntReg3,SoftQL,TRPO}. These regularizers are convex functions, hence encouraging the policy to be stochastic. Traditional wisdom behind using policy regularization such as policy entropy is that it encourages stochasticity in policy and hence improves exploration while learning. The above result gives another twist to the story, as it clearly indicates that policy regularizers help in robustness and hence may improve generalization. In literature, the above regularizer is scaled with a constant mostly chosen arbitrarily, here we have the explicit constant $\alpha_s +\gamma\beta_{s}\kappa_q(v)$ that depends on uncertainty radiuses and variance of the value function. 

\begin{theorem}\label{rs:rve}(Policy improvement) The optimal robust Bellman operator $(\mathcal{T}^*_{\mathcal{U}^s_p}v)(s)$ satisfies
\begin{align}\label{eq:rve:val}
    \sum_{a}\bigm(Q(s,a) - x\bigm)^p\mathbf{1}\bigm( Q(s,a) \geq x\bigm)  = \sigma^p,
\end{align}
 that can be found using binary search between $\bigm[\max_{a}Q(s,a)-\sigma, \max_{a}Q(s,a)\bigm]$,
where $\sigma = \alpha_s + \gamma\beta_s\kappa_q(v)$, and  $Q(s,a) = R_0(s,a) + \gamma\sum_{s'} P_0(s'|s,a)v(s')$. 
\end{theorem}
\begin{proof}
The proof  has the following main steps,
\begin{align*}
 &(\mathcal{T}^*_{\mathcal{U}^s_p}v)(s)= \max_{\pi_s\in\Delta_\mathcal{A}}\min_{{R,P\in\mathcal{U}^s_p}}\sum_{a}\pi_s(a)\Bigm[R(s,a) +  \gamma \sum_{s'}P(s'|s,a)v(s')\Bigm], \qquad \text{(from definition)}\\
 &\text{(Solving inner loop (policy evaluation) using lemma \ref{rs:SLpPlanning})}\\
 &=  \max_{\pi_s\in\Delta_\mathcal{A}}\Bigm[-\bigm(\alpha_s +\gamma\beta_{s}\kappa_q(v)\bigm)\lVert\pi_s\rVert_q  +\sum_{a}\pi_s(a)\bigm(R_0(s,a) +\gamma \sum_{s'}P_0(s'|s,a)v(s')\bigm)\Bigm]\\
 &= \max_{\pi_s\in\Delta_\mathcal{A}}- \alpha\lVert \pi_s\rVert_q + \langle \pi_s,b\rangle \qquad \text{( for appropriate  $\alpha, b$).}
\end{align*}
The rest follows from the solution of the above optimization problem. The proof of the second part is more technical so referred to the appendix. For $p=2$, the proof can be found in \cite{anava2016k}.
\end{proof}
The solution to \eqref{eq:rve:val} can be found via algorithm \ref{alg:fp} too discussed in appendix \ref{app:properties}. Notably, it can be exactly solved for $p=1$ (algorithm \ref{alg:f1} in appendix) and algorithm \ref{alg:main:f2}. And for $p=1,\infty$, solution to \eqref{eq:rve:val} can be obtained in closed form. The best of both worlds are summarized in table \ref{tb:val}. This implies that the exact policy improvement step can be done for $p=1,2,\infty$ in closed form or by a simple algorithm and inexactly for general $p$ via binary search efficiently.

\begin{algorithm}[tb]
\caption{$f_2(x,\sigma)$}
\label{alg:main:f2}
\textbf{Input}: $x,\sigma$
\begin{algorithmic}[1] 
\STATE Sort $x$ such that $x_1\geq x_2, \cdots \geq x_A$.
\STATE Set $k=0$ and $\lambda = x_1-\sigma$

\WHILE{$k \leq A-1  $ and $\lambda \leq x_k$}
    \STATE  $k = k+1$
    \STATE  \[\lambda = \frac{1}{k}\bigm[\sum_{i=1}^{k}x_i - \sqrt{k\sigma^2 + (\sum_{i=1}^{k}x_i^2 - k\sum_{i=1}^{k}x_i)^2}\Bigm]\]
\ENDWHILE

\STATE \textbf{return} $\lambda$
\end{algorithmic}
\end{algorithm}

\begin{theorem}\label{rs:srect:optimalPolicy} (Optimal policy) The optimal policy  $\pi = \pi^*_{\mathcal{U}^s_p}$ is a threshold policy, that is proportional to the advantage function, that is 
    \[ \pi(a|s) \propto \bigm( A(s,a)\bigm )^{p-1}\mathbf{1} \bigm( A(s,a)\geq  0\bigm ),\]
    where $A(s,a)= Q^*_{\mathcal{U}^s_p}(s,a)-v^*_{\mathcal{U}^s_p}(s)$ and $ Q^*_{\mathcal{U}^s_p}(s,a) = R_0(s,a) + \gamma\sum_{s'}P_0(s'|s,a) v^*_{\mathcal{U}^s_p}(s)$.
\end{theorem}
The above policy only plays some top actions, as a contrast to soft Q-learning (with entropy regularization)\cite{SoftQL,EntReg1,TRPO} where optimal policy is of the form $\pi(a|s) \propto e^{\eta (Q(a|s)-v(s))}$. Like the softmax policy, these policies give more weight to actions having more advantages but unlike softmax, it doesn't play utterly useless actions. To the best of our knowledge, this kind of policy is novel in literature.\\

The notable special cases of the above theorem for $p=1,2,\infty$, are summarized along with others in table \ref{tb:Opt:pi}.  Model-based Q-learning algorithm based on the above results is presented in appendix \ref{app:ModelBasedAlgorithms}.  Sample-based algorithm for \texttt{s}-rectangular $L_2$ robust MDPs, is implemented in algorithm \ref{alg:SL2}.
 
\begin{table*}
\caption{Optimal Policy}
  \label{tb:Opt:pi}
  \centering
  \begin{tabular}{lll}
    \toprule                   
    $\mathcal{U}$     & $\pi^*_\mathcal{U}(a|s) \propto$      & remark \\
    \midrule
    $\mathcal{U}^s_p$ &$ (A^*_\mathcal{U}(s,a))^{p-1}\mathbf{1}(A^*_\mathcal{U}(s,a) \geq 0)$ &  top actions proportional to\\ & &  $(p-1)$th power of advantage\\
    &\\
    $\mathcal{U}^s_1$    &$\mathbf{1}(A^*_\mathcal{U}(s,a) \geq 0)$&top  actions with uniform probability  \\
    &\\
    $\mathcal{U}^s_2$     &$(A^*_\mathcal{U}(s,a))\mathbf{1}(A^*_\mathcal{U}(s,a) \geq 0)$ & top   actions proportion to advantage \\ 
    &\\
    $\mathcal{U}^s_\infty$     &$  arg\max_{a\in\mathcal{A}}Q^*_\mathcal{U}(s,a)$&  best action \\ &\\
    $\mathcal{U}^{\mathtt{sa}}_p$    &$  arg\max_{a}[- \alpha_{sa} -\gamma\beta_{sa} \kappa_q(v^*_\mathcal{U}) +Q^*_\mathcal{U}(s,a)]$&  best regularized action \\
    \bottomrule
 \end{tabular}
  \begin{tabular}{l}
  where $Q^*_\mathcal{U}(s,a) = R_0(s,a) + \gamma\sum_{s'}P_0(s'|s,a)v^*_\mathcal{U}(s')$, and $A^*_\mathcal{U}(s,a) = Q^*_\mathcal{U}(s,a)-v^*_\mathcal{U}(s)$.
  \end{tabular}
\end{table*}

\begin{algorithm}[H]
\caption{Model Free Algorithm for \texttt{s}-rectangular $L_2$ robust MDPs}
\label{alg:SL2}
\textbf{Input}:$\alpha_{s},\beta_{s},s_0 \sim \mu,\eta_n,\eta'_n,\epsilon_n$. Initialize $Q,v$ randomly, and $\mu_0,\rho_0,\kappa,n=0$.
\begin{algorithmic}[1] 
\WHILE{ not converged}
    \STATE Update reward regularizer.
    $$\sigma =\alpha_{s_n} +\gamma\beta_{s_n}\kappa $$
    \STATE Value update: Using  Algorithm \ref{alg:main:f2}, get
    \[v(s_n) =v(s_n) + \eta_n[f_2(Q(s_n,\cdot),\sigma)-v(s_n)]  \]
    \STATE Greedy policy:
    \[\pi(a|s_n)  \propto (Q(s_n,a)-v(s_n))\mathbf{1}(Q(s_n,a) \geq v(s_n))\]    
    \STATE With probability $1-\epsilon_n$, play optimal policy $$a_n \sim \pi(\cdot|s_n)$$ 
    and with probability $\epsilon_n$ play exploratory action.\\
    \STATE Get next state $s_{n+1}$ from the environment.
    \STATE Update Q-value as 
    \begin{align*}
        Q(s_n,a_n) = &Q(s_n,a_n) + \eta'_n\Bigm[R(s_n,a_n) +\gamma v(s_{n+1})-Q(s_n,a_n)\Bigm]
    \end{align*} 
    \STATE Update the value mean (first moment) 
    $$\mu_{n+1} =\mu_n + \frac{v(s_n) -\mu_n}{n+1} $$
    \STATE Update the value second-moment 
    \[\rho_{n+1} =\rho_n + (v(s_n))^2 \]
    \STATE Update $p$-value variance
    \[\kappa= \frac{S}{n+1}\sqrt{\rho_{n+1}-(n+1)\mu_{n+1}^2}\]
\ENDWHILE

\STATE \textbf{return} $v,\pi$
\end{algorithmic}
\end{algorithm}



\begin{table*} 
  \caption{Relative running cost (time) for value iteration}
  \centering
  \begin{tabular}{lp{0.5cm}p{0.7cm}p{0.7cm}lllllllll}
    \toprule                   
    S&A& non-robust& $\mathcal{U}^{sa}_1$ LP&$\mathcal{U}^{s}_1$ LP&$\mathcal{U}^{sa}_1$&$\mathcal{U}^{sa}_2$ &$\mathcal{U}^{sa}_\infty$&$\mathcal{U}^{s}_1$&$\mathcal{U}^{s}_2$&$\mathcal{U}^{s}_\infty$&$\mathcal{U}^{sa}_{10}$&$\mathcal{U}^{s}_{10}$\\
     \midrule
    10&10&1&1438&72625&1.7&1.5&1.5&1.4&2.6&1.4&5.5&33\\
   30&10&1&6616&629890&1.3&1.4&1.4&1.5&2.8&3.0&5.2&78\\   
    50&10&1&6622&4904004&1.5&1.9&1.3&1.2&2.4&2.2&4.1&41\\
    100&20&1&16714&NA&1.4&1.5&1.5&1.1&2.1&1.5&3.2&41\\
    \bottomrule
  \end{tabular}
  \end{table*}
  
\section{Time complexity }
Here we study the time complexity of the value iteration of algorithms (algorithm \ref{alg:SALp} and algorithm \ref{alg:SLp}) for  $L_p$ robust MDPs (when nominal reward and nominal transition kernel are known) and compare it to the non-robust counterpart. Precisely, we compare the number of iterations required for value iteration to converge to $\epsilon$ close to the optimal value function for different robust MDPs. \\
Value (or Q-value) iterations discussed above, have mainly two parts where they differ: a) reward cost: effective reward calculation (essentially $\kappa_p$ )  b) action cost: choosing the optimal actions. 
Non-robust MDPs require constant time to calculate reward as there is no penalty, and it takes $O(A)$ time to compute action as we need to find the best action.
For $p=1,2,\infty$, calculation of $\kappa_p$ and evaluation of robust Bellman operator is done exactly. Different robust MDPs require different operations such as sorting of actions, sorting of value functions, calculation of best action, variance of value function,etc that accounts for different complexity. For general $p$, both evaluation of $\kappa_p$ and robust Bellman operator is done approximately through binary search. That leads to increased complexity, yet the increase is not significant. The results are summarized in the comparison table \ref{tb:time} and the proofs are in appendix \ref{app:timeComplexitySection}. Comparing the complexity in table \ref{tb:time}, shows that the $L_p$ robust MDPs are not harder than non-robust MDPs. It is noteworthy to see that at the limit $S\to \infty$ (keeping action space $A$ and tolerance $\epsilon$ constant), the complexity of all robust MDPs is the same as non-robust MDPs. It is also confirmed by experiments.  
\begin{table}
  \caption{Time complexity}
  \label{tb:time}
  \centering
  \begin{tabular}{ll}
    \toprule            
    & Total cost $O$ \\ 
   \midrule
 Non-Robust MDP & $\log(1/\epsilon)S^2A$\\
 $\mathcal{U}^{\mathtt{sa}}_1$& $\log(1/\epsilon)S^2A$\\
 $\mathcal{U}^{\mathtt{sa}}_2$ &$\log(1/\epsilon)S^2A$\\
 $\mathcal{U}^{\mathtt{sa}}_\infty$&$\log(1/\epsilon)S^2A$\\
  $\mathcal{U}^{\mathtt{s}}_1$& $\log(1/\epsilon)(S^2A + SA\log(A))$\\
 $\mathcal{U}^{\mathtt{s}}_2$&$\log(1/\epsilon)(S^2A + SA\log(A))$\\
 $\mathcal{U}^{\mathtt{s}}_\infty$ &$\log(1/\epsilon)S^2A$\\
 \midrule
 $\mathcal{U}^{\mathtt{sa}}_p$& $\log(1/\epsilon)\bigm(S^2A + S\log(S
 /\epsilon)\bigm) $\\
 $\mathcal{U}^{\mathtt{s}}_p$&$\log(1/\epsilon)\bigm( S^2A+ SA\log(A/\epsilon) \bigm) $\\
 \bottomrule
 \end{tabular}
\end{table}

\section{Experiments}\label{app:experiments}
Table 4 and figure \ref{fig:asym} contains the relative cost (time) of robust value iteration w.r.t. non-robust MDP, for randomly generated kernel and reward function with the number of states $S$ and the number of actions $A$. Table 4 shows that \texttt{s} and \texttt{sa}-rectangular MDPs are indeed hard using numerical methods such as Linear Programming (LP), and our methods perform as well as non-robust MDPs, particularly for $p=1,2,\infty.$ For general $p$, we require binary search that requires $30-50$ iterations for acceptable tolerance, hence it takes more time than $p=1,2,\infty$ where it can done exactly.
\begin{figure}
      \centering
       \includegraphics[width=8cm, height=5cm]{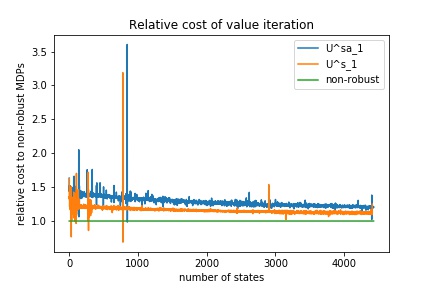}
      \caption{Relative cost of value iteration w.r.t. non-robust MDP at different $S$ with fixed $A=10$.}
      \label{fig:asym}
  \end{figure}
As we can see from our complexity analysis, the relative cost of value iteration converges to 1 as the number of states increases keeping the number of actions fixed also confirmed by figure \ref{fig:asym}.
 The rate of convergence for all the above settings was the same as the non-robust one, as predicted by the theory. The experiments run few runs hence there is some stochasticity in the results, but the trend is clear. We refer the reader to section \ref{app:experiments} for more details. All codes and results are available at https://github.com/navdeepkumar12/rmdp.

\section{Conclusion and future work}\label{conclusion}
This work develops methodology for policy improvement for \texttt{s}-rectangular $L_p$ robust MDPs, and fully characterizes the greedy policy. It also corrects the noise condition in transition kernel that leads to a different reward regularizer that depends on $q$-variance of the value function for $L_p$ robust MDPs. The work concludes that  $L_p$ robust MDPs are computationally as easy as non-robust MDPs, and adds a concrete link between regularized MDPs and $L_p$ robust MDPs. Further, the $p$-mean function $\kappa_p$ can be estimated in an online fashion. This paves the path for sample-based (model-free) algorithms for (\texttt{sa})/\texttt{s}-rectangular $L_p$ robust MDPs that can be implemented using deep neural networks (such as DQN). We studied the $L_p$ constrained rectangular uncertainty set, we leave for future work to extend this work to the other uncertainty sets.

\bibliography{main}
\bibliographystyle{plain}

\section{How to read appendix}
\begin{enumerate}

    \item Section \ref{app:relatedWork} contains related work.
    \item Section \ref{app:properties} contains additional properties and results that couldn't be included in the main section for the sake of clarity and space. Many of the results in the main paper is special of results in this section. 
    \item Section \ref{app:ForbiddenTransition} contains the discussion on zero transition kernel (forbidden transitions).
    \item Section \ref{app:UCRL} contains a possible connection this work to UCRL.
    \item Section \ref{app:experiments} contains additional experimental results and a detailed discussion.
     \item All the proofs of the main body of the paper is presented in the section  \ref{app:srLp} and \ref{app:timeComplexitySection}.
    \item Section \ref{app:pvarianceSection} contains helper results for section \ref{app:srLp}. Particularly, it discusses $p$-mean function $\omega_p$ and $p$-variance function $\kappa_p$.
    \item Section \ref{app:waterPouringSection} contains helper results for section \ref{app:srLp}. Particularly, it discusses $L_p$ water pouring lemma, necessary to evaluate robust optimal Bellman operator (learning) for $\mathtt{s}$-rectangular $L_p$ robust MDPs.
    \item Section \ref{app:timeComplexitySection} contains time complexity proof for model based algorithms.
    \item Section \ref{app:SALpQL} develops Q-learning machinery for $(\mathtt{sa})$-rectangular $L_p$ robust MDPs based on the results in the main section. It is not used in the main body or anywhere else, but this provides a good understanding for algorithms proposed in section \ref{app:ModelBasedAlgorithms} for $(\mathtt{sa})$-rectangular case.
    \item Section \ref{app:ModelBasedAlgorithms} contains model-based algorithms for $\mathtt{s}$ and $(\mathtt{sa})$-rectangular $L_p$ robust MDPs. It also contains, remarks for special cases for $p=1,2,\infty$.
\end{enumerate}
\section{Related Work} \label{app:relatedWork}
\subsection*{ R-Contamination Uncertainty Robust MDPs}
The paper \cite{Rcontamination} considers the following uncertainty set for some fixed constant $0 \leq R \leq 1$,
\begin{align}
    \mathcal{P}_{sa} = \{ (1-R)(P_0)(\cdot|s,a) + RP \mid P \in \Delta_{\mathcal{S}} \}, \quad s\in \mathcal{S}, a\in\mathcal{A},
\end{align}
and  $\mathcal{P} = \otimes_{s,a}\mathcal{P}_{s,a}, \qquad \mathcal{U} = \{R_0\}\times\mathcal{P}$. The robust value function $v^\pi_\mathcal{U}$ is the fixed point of the robust Bellman operator defined as 
\begin{align}
    (\mathcal{T}^\pi_\mathcal{U}v)(s) :=& \min_{P\in\mathcal{P}}\sum_{a}\pi(a|s)[R_0(s,a) + \gamma \sum_{s'}P(s'|s,a)v(s')],\\
    =&\sum_{a}\pi(a|s)[R_0(s,a) - \gamma R\max_{s}v(s) + (1-R)\gamma \sum_{s'}P_0(s'|s,a)v(s')].
\end{align}
And the optimal robust value function $v^*_\mathcal{Ua }$ is the fixed point of the optimal robust Bellman operator defined as 
\begin{align}
    (\mathcal{T}^*_\mathcal{U}v)(s) :=& \max_{\pi}\min_{P\in\mathcal{P}}\sum_ {a}\pi(a|s)[R_0(s,a) + \gamma (1-R)\sum_{s'}P(s'|s,a)v(s')],\\
    =&\max_{a}[R_0(s,a) - \gamma R\max_{s}v(s) + \gamma (1-R)\sum_{s'}P_0(s'|s,a)v(s')].
\end{align}
Since, the uncertainty set is \texttt{sa}-rectangular, hence the map is a contraction \cite{Nilim2005RobustCO}, so the robust value iteration here, will also converge linearly similar to non-robust MDPs. It is also possible to obtain Q-learning as following
\begin{align}
    Q_{n+1}(s,a) = R_0(s,a)  - \gamma R\max_{s,a}Q_n(s,a) + \gamma(1-R) \sum_{s'}P_0(s'|s,a)\max_{s'}Q_n(s',a').
\end{align}
Convergence of the above Q-learning follows from the contraction of robust value iteration. Further, it is easy to see that model-free Q-learning can be obtained from the above.\\

A follow-up work \cite{PG_RContamination} proposes a policy gradient method for the same.
\begin{proposition}(Theorem 3.3 of \cite{PG_RContamination}) Consider a class of policies $\Pi$ satisfying Assumption 3.2 of \cite{PG_RContamination}. The gradient of the robust return is given by
\begin{align*}
    \nabla \rho^{\pi_\theta} = &\frac{\gamma R}{(1-\gamma)(1-\gamma +\gamma R)}\sum_{s,a}d^{\pi_\theta}_\mu(s,a)\nabla{\pi_\theta}(a|s)Q^{\pi_\theta}_{\mathcal{U}}(s,a)  \\
    &\qquad +  \frac{1}{1-\gamma +\gamma R}\sum_{s,a}d^{\pi_\theta}_{s_\theta}(s,a)\nabla{\pi_\theta}(a|s)Q^{\pi_\theta}_{\mathcal{U}}(s,a) ,
\end{align*}
where $s_\theta\in arg\max v^{\pi_\theta}_{\mathcal{U}}(s) $, and $Q^\pi_\mathcal{U}(s,a) = \sum_{a}\pi(a|s)\bigm[R_0(s,a) - \gamma R \max_{s}v^\pi_\mathcal{U}(s) + \gamma (1-R)\sum_{s'}P_0(s'|s,a)v^\pi_\mathcal{U}(s')\bigm].$
\end{proposition}
The work shows that the proposed robust policy gradient method converges to the global optimum asymptotically under direct policy parameterization.\\

The uncertainty set considered here, is \texttt{sa}-rectangular, as uncertainty in each state-action is independent, hence the regularizer term $(\gamma R\max_{s}v(s))$ is independent of policy, and the optimal (and greedy) policy is deterministic. It is unclear, how the uncertainty set can be generalized to the $s$-rectangular case. Observe that the above results resemble very closely our \texttt{sa}-rectangular $L_1$ robust MDPs results.\\

\subsection*{Twice Regularized MDPs}
The paper \cite{derman2021twice} paper robust MDPs to twice regularized MDPs, and proposes a policy gradient method for solving them. 

\begin{proposition} (corollary 3.1 of \cite{derman2021twice}) (\texttt{s}-rectangular reward robust policy evaluation) Let the uncertainty set be $\mathcal{U} = (R_0+\mathcal{R})\times\{P_0\}$, where $\mathcal{R}_s =\{ r_s\in\mathbf{R}^{\mathcal{A}}\mid \lVert r_s\rVert \leq \alpha_{s}\}$ for all $s\in\mathcal{S}$. Then the robust value function $v^\pi_{\mathcal{U}}$ is the optimal solution to the convex optimization problem:
\[ \max_{v\in\mathbf{R}^\mathcal{A}}\langle \mu,v\rangle \quad s.t.\quad v(s)\leq (\mathcal{T}^\pi_{R_0,P_0}v)(s) -\alpha_{s}\lVert\pi_s\rVert, \qquad \forall s\in\mathcal{S}.\]
\end{proposition}
It derives the policy gradient for reward robust MDPs to obtain the optimal robust policy $\pi^*_\mathcal{U}$.

\begin{proposition} (Proposition 3.2 of \cite{derman2021twice}) (\texttt{s}-rectangular reward robust policy gradient) Let the uncertainty set be $\mathcal{U} = (R_0+\mathcal{R})\times\{P_0\}$, where $\mathcal{R}_s =\{ r_s\in\mathbf{R}^{\mathcal{A}}\mid \lVert r_s\rVert \leq \alpha_{s}\}$ for all $s\in\mathcal{S}$. Then the gradient of the reward robust objective $ \rho^\pi_{\mathcal{U}} := \langle \mu, v^\pi_{\mathcal{U}}\rangle$ is given by
\[ \nabla \rho^\pi_\mathcal{U} = \mathbf{E}_{(s,a)\sim d^\pi_{P_0}}\Bigm[ \nabla\ln(\pi(a|s))\bigm(Q^\pi_{\mathcal{U}}(s,a) -\alpha_{s}\frac{\pi(a|s)}{\lVert \pi_s \rVert}\bigm)\Bigm],\]
where $Q^\pi_{\mathcal{U}}(s,a):=\min_{(R,P)\in\mathcal{U}}[R(s,a)+\gamma \sum_{s'}P(s'|s,a)v^\pi_{\mathcal{U}}(s')].$
\end{proposition}

\begin{proposition} (Corollary 4.1 of \cite{derman2021twice}) (\texttt{s}-rectangular general robust policy evaluation) Let the uncertainty set be $\mathcal{U} = (R_0+\mathcal{R})\times\{P_0+\mathcal{P}\}$, where $\mathcal{R}_s =\{ r_s\in\mathbf{R}^{\mathcal{A}}\mid \lVert r_s\rVert \leq \alpha_{s}\}$  and $\mathcal{P}_s =\{ P_s\in\mathbf{R}^{\mathcal{S}\times\mathcal{A}}\mid \lVert P_s\rVert \leq \beta_{s}\}$ for all $s\in\mathcal{S}$. Then the robust value function $v^\pi_{\mathcal{U}}$ is the optimal solution to the convex optimization problem:
\[ \max_{v\in\mathbf{R}^\mathcal{A}}\langle \mu,v\rangle \quad s.t.\quad v(s)\leq (\mathcal{T}^\pi_{R_0,P_0}v)(s) -\alpha_{s}\lVert\pi_s\rVert -\gamma\beta_{s}\lVert v\rVert\lVert\pi_s\rVert, \qquad \forall s\in\mathcal{S}.\]
\end{proposition}
Same as the reward robust case, the paper tries to find a policy gradient method to obtain the optimal robust policy. Unfortunately, the dependence of regularizer terms on value makes it a very difficult task. Hence it proposes the $R^2$MPI algorithm (algorithm 1 of ~\cite{derman2021twice}) for the purpose that optimizing the greedy step via projection onto the simplex using a black box solver. Note that the above proposition is not same as our policy evaluation (although it looks similar), it requires some extra assumptions (assumption 5.1 ~\cite{derman2021twice}) and lot of work ensure $R^2$ Bellman operator is contraction etc. In our case, we directly evaluate robust Bellman operator that has already proven to be a contraction, hence we don't require any extra assumption nor any other work as ~\cite{derman2021twice}. 

Our work makes improvements over this work by explicitly solving both policy evaluation and policy improvement in general robust MDPs. It also makes more realistic assumptions on the transition kernel uncertainty set. 

\subsection*{Regularizer solves Robust MDPs}
The work \cite{MaxEntRL} looks in the opposite direction than we do. It investigates the impact of the popularly used entropy regularizer on robustness. It finds that MaxEnt can be used to maximize a lower bound on a certain robust RL objective (reward robust).  \\

As we noticed that $\lVert \pi_s \rVert_q$ behaves like entropy in our regularization. Further, our work also deals with uncertainty in transition kernel in addition to the uncertainty in reward function.

\subsection*{Upper Confidence RL}
The upper confidence setting in \cite{UCRL,UCRL2} is very similar to our $L_p$ robust setting. We refer to this discussion in section \ref{app:UCRL}.

\section{S-rectangular: More Properties}\label{app:properties}
 
\begin{definition} We begin with the following notational definitions.
    \begin{enumerate}
        \item  $Q$-value at value function $v$ is defined as 
        \[Q^v(s,a) := R_0(s,a) + \gamma\sum_{s'}P_0(s'|s,a) v(s').\]
        \item Optimal $Q$-value is defined as 
        \[Q^*_{\mathcal{U}}(s,a) = R_0(s,a) + \gamma\sum_{s'}P_0(s'|s,a) v^*_{\mathcal{U}}(s')\]
        \item With little abuse of notation, $Q(s,a_i)$ shall denote the $i$th best value in state $s$, that is
        \[Q(s,a_1)\geq Q(s,a_2)\geq,\cdots,\geq Q(s,a_A).\]
        \item $\pi^v_\mathcal{U}$ denotes the greedy policy at value function $v$, that is
         \[\mathcal{T}^{*}_{\mathcal{U}}v = \mathcal{T}^{\pi^v_{\mathcal{U}}}_{\mathcal{U}}v. \]
         \item $\chi_p(s)$ denotes the number of active actions in state $s$ in \texttt{s}-rectangular $L_p$ robust MDPs, defined as
        \[\chi_p(s) :=\bigm\lvert \{a \mid \pi^*_{\mathcal{U}^s_p}(a|s) \geq 0\}\bigm\rvert .\]
        \item $\chi_p(p,s)$ denotes the number of active actions in state $s$ at value function $v$ in \texttt{s}-rectangular $L_p$ robust MDPs, defined as
        \[\chi_p(v,s) :=\bigm\lvert \{a \mid \pi^v_{\mathcal{U}^s_p}(a|s) \geq 0\}\bigm\rvert .\]
    \end{enumerate}
\end{definition}
We saw above that optimal policy in \texttt{s}-rectangular robust MDPs may be stochastic. The action that has a positive advantage is active and the rest are inactive. Let $\chi_p(s)$ be the number of active actions in state $s$, defined as 
\begin{align}
    \chi_p(s) :=\bigm\lvert \{a \mid \pi^*_{\mathcal{U}^s_p}(a|s) \geq 0\}\bigm\rvert =\bigm\lvert \{a \mid Q^*_{\mathcal{U}^s_p}(s,a) \geq v^*_{\mathcal{U}^s_p}(s)\}\bigm\rvert.
\end{align}
Last equality comes from theorem \ref{rs:srect:optimalPolicy}. One direct relation between Q-value and value function is given by 
\begin{align}
    v^*_{\mathcal{U}^s_p}(s) = \sum_{a}\pi^*_{\mathcal{U}^s_p}(a|s)\Bigm[-\bigm(\alpha_s +\gamma\beta_{s}\kappa_q(v)\bigm)\lVert\pi^*_{\mathcal{U}^s_p}(\cdot|s)\rVert_q  +Q^*_{\mathcal{U}^s_p}(s,a)\Bigm].
\end{align}
The above relation is very convoluted compared to non-robust and  \texttt{sa}-rectangular robust cases. The property below illuminates an interesting relation.
\begin{property} \label{rs:Opt:EVfQv}(Optimal Value vs Q-value) $v^*_{\mathcal{U}^s_p}(s)$ is bounded by the Q-value of $\chi_p(s)$th and $(\chi_p(s)+1)$th actions, that is , 
\[Q^*_{\mathcal{U}^s_p}(s, a_{\chi_p(s)+1}) < v^*_{\mathcal{U}^s_p}(s) \leq Q^*_{\mathcal{U}^s_p}(s, a_{\chi_p(s)}).\]
\end{property}
This special case of the property \ref{rs:EVfQv}, similarly table \ref{tb:Opt:ValVsQ} is special case of table \ref{tb:ValVsQ}.
\begin{table}[H]
\caption{Optimal value function and  Q-value}
  \label{tb:Opt:ValVsQ}
  \centering
  \begin{tabular}{lll}
    \toprule                   
     $v^*(s) = \max_{a}Q^*(s,a) $      & Best value& \\
    $ v^*_{\mathcal{U}^{sa}_p}(s) = \max_a[\alpha_{s,a}-\gamma\beta_{s,a}\kappa_q(v^*_{\mathcal{U}^{sa}_p})-Q^*_{\mathcal{U}^{sa}_p}(s,a) ]$ & Best regularized value&\\
    $ Q^*_{\mathcal{U}^s_p}(s, a_{\chi_p(s)+1}) < v^*_{\mathcal{U}^s_p}(s) \leq Q^*_{\mathcal{U}^s_p}(s, a_{\chi_p(s)})$ & Sandwich!&\\
    \bottomrule
 \end{tabular}
  \begin{tabular}{l}
  where $v^*, Q^*$ is the optimal value function and Q-value respectively 
  of non-robust MDP.
  \end{tabular}
\end{table}
The same is true for the non-optimal Q-value and value function.

\begin{theorem}\label{rs:srect:greddyPolicy} (Greedy policy) The greedy policy  $\pi^v_{\mathcal{U}^s_p}$ is a threshold policy, that is proportional to the advantage function, that is 
    \[ \pi^v_{\mathcal{U}^s_p}(a|s) \propto \bigm( Q^v(s,a)-(\mathcal{T}^*_{\mathcal{U}^s_p}v)(s)\bigm )^{p-1}\mathbf{1} \bigm(  Q^v(s,a)\geq (\mathcal{T}^*_{\mathcal{U}^s_p}v)(s)\bigm ).\]
   
\end{theorem}
The above theorem is proved in the appendix, and theorem \ref{rs:srect:optimalPolicy} is its special case. So is table \ref{tb:Opt:pi} special case of table \ref{tb:pi}.
\begin{table}[H]
\caption{Greedy policy at value function $v$}
  \label{tb:pi}
  \centering
  \begin{tabular}{lll}
    \toprule                   
    $\mathcal{U}$     & $\pi^v_\mathcal{U}(a|s) \propto$      & remark \\
    \midrule
    $\mathcal{U}^s_p$ &$ (Q^v(s,a)- (\mathcal{T}^*_{\mathcal{U}}v)(s))^{p-1}\mathbf{1}(A^v_\mathcal{U}(s,a)\geq0)$ &  top actions proportional to\\ & &  $(p-1)$th power of its advantage\\
    &\\
    $\mathcal{U}^s_1$    &$\frac{ \mathbf{1}(A^v_\mathcal{U}(s,a)\geq0)}{\sum_{a}\mathbf{1}(A^v_\mathcal{U}(s,a)\geq0)}$&top  actions with uniform probability  \\
    &\\
    $\mathcal{U}^s_2$     &$\frac{A^v_\mathcal{U}(s,a)\mathbf{1}A^v_\mathcal{U}(s,a)\geq0)}{\sum_{a}A^v_\mathcal{U}(s,a)\mathbf{1}(A^v_\mathcal{U}(s,a)\geq0)}$ & top   actions proportion to its advantage \\ 
    &\\
    $\mathcal{U}^s_\infty$     &$  arg\max_{a\in\mathcal{A}}Q^v(s,a)$&  best action \\ &\\
    $\mathcal{U}^{\mathtt{sa}}_p$    &$  arg\max_{a}[- \alpha_{sa} -\gamma\beta_{sa} \kappa_q(v) +Q^v(s,a)]$&  best action \\
    \bottomrule
 \end{tabular}
  \begin{tabular}{l}
  where $A^v_\mathcal{U}(s,a) = Q^v(s,a) -(\mathcal{T}^*_\mathcal{U}v)(s)$ and $Q^v(s,a) = R_0(s,a) + \gamma\sum_{s'}P_0(s'|s,a)v(s')$.
  \end{tabular}
\end{table}
The above result states that the greedy policy takes actions that have a positive advantage, so we have.
\begin{align}
    \chi_p(v,s) :=\bigm\lvert \{a \mid \pi^v_{\mathcal{U}^s_p}(a|s) \geq 0\}\bigm\rvert =\bigm\lvert \{a \mid Q^v(s,a) \geq (\mathcal{T}^*_{\mathcal{U}^s_p})v(s)\}\bigm\rvert.
\end{align}
\begin{property} \label{rs:EVfQv}(Greedy Value vs Q-value) $(\mathcal{T}^*_{\mathcal{U}^s_p}v)(s)$ is bounded by the Q-value of $\chi_p(v,s)$th and $(\chi_p(v,s)+1)$th actions, that is , 
\[Q^v(s, a_{\chi_p(v,s)+1}) < (\mathcal{T}^*_{\mathcal{U}^s_p}v)(s) \leq Q^v(s, a_{\chi_p(v,s)}).\]
\end{property}
\begin{table}[H]
\caption{ Greedy value function and  Q-value}
  \label{tb:ValVsQ}
  \centering
  \begin{tabular}{ll}
    \toprule                   
     $(\mathcal{T}^*v)(s) = \max_{a}Q^v(s,a) $      & Best value \\
    $ (\mathcal{T}^*_{\mathcal{U}^{sa}_p})v(s) = \max_a[\alpha_{s,a}-\gamma\beta_{s,a}\kappa_q(v)-Q^v(s,a) ]$ & Best regularized value\\
    $ Q^v(s, a_{\chi_p(v,s)+1}) < (\mathcal{T}^*_{\mathcal{U}^s_p})v(s) \leq Q^v(s, a_{\chi_p(v,s)})$ & Sandwich!\\
    \bottomrule
 \end{tabular}
  \begin{tabular}{l}
  where   $Q^v(s,a_1)\geq,\cdots,\geq Q^v(s,a_A).$
  \end{tabular}
\end{table}
The property below states that we can compute the number of active actions $\chi_p(v,s)$ (and $\chi_p(s)$) directly without computing greedy (optimal) policy.
\begin{property} $\chi_p(v,s)$ is number of actions that has positive advantage, that is 
 \[\chi_p(v,s) := \max\{k\mid \sum_{i=1}^k\bigm(Q^v(s,a_i) - Q^v(s,a_k)\bigm)^p  \leq \sigma^p \},\]
 where $\sigma = \alpha_s + \gamma\beta_s\kappa_q(v)$, and $Q^v(s,a_1)\geq Q^v(s,a_2),\geq \cdots \geq Q(s,a_A).$ 
\end{property}
When uncertainty radiuses ($\alpha_s,\beta_s$) are zero (essentially $\sigma =0$ ), then $\chi_p(v,s) = 1, \forall v,s$, that means, greedy policy taking the best action. In other words, all the robust results reduce to non-robust results as discussed in section \ref{sec:MDPs} as the uncertainty radius becomes zero.

\begin{algorithm}[H]
\caption{Algorithm to compute \texttt{s}-rectangular $L_p$ robust optimal Bellman Operator}\label{alg:fp}
\begin{algorithmic} [1]
 \STATE \textbf{Input:} $\sigma = \alpha_s +\gamma\beta_s\kappa_q(v), \qquad Q(s,a) = R_0(s,a) + \gamma\sum_{s'} P_0(s'|s,a)v(s')$.
 \STATE \textbf{Output} $(\mathcal{T}^*_{\mathcal{U}^s_p}v)(s), \chi_p(v,s)$
\STATE Sort $Q(s,\cdot)$ and label actions such that $Q(s,a_1)\geq Q(s,a_2), \cdots$.
\STATE Set initial value guess $\lambda_1 = Q(s,a_1)-\sigma$ and counter $k=1$.
\WHILE{$k \leq A-1  $ and $\lambda_k \leq Q(s,a_k)$}
    \STATE Increment counter: $k = k+1$
    \STATE Take $\lambda_k$ to be a solution of the following  
    \begin{equation}\label{eq:alg:fp:lk}
    \sum_{i=1}^{k}\bigm(Q(s,a_i) - x\bigm)^p  = \sigma^p, \quad\text{and}\quad x\leq Q(s,a_k).    
    \end{equation}
\ENDWHILE
\STATE Return: $\lambda_k, k $
\end{algorithmic}
\end{algorithm}
\section{Revisiting kernel noise assumption} \label{app:ForbiddenTransition}
\subsection*{Sa-Rectangular Uncertainty}
Suppose at state $s$, we know that it is impossible to have transition (next) to some states (forbidden states $F_{s,a}$) under some action. That is, we have the transition uncertainty set $\mathcal{P}$ and nominal kernel $P_0$  such that 
\begin{align}
   P_0(s'|s,a) =  P(s'|s,a) = 0, \quad \forall P\in\mathcal{P},\forall s' \in F_{s,a}.
\end{align}
Then we define, the kernel noise as 
\begin{align}
    \mathcal{P}_{s,a} = \{ P \mid \lVert P\rVert_p = \beta_{s,a}, \quad \sum_{s'}P(s')=0, \quad P(s")=0, \forall s"\in F_{s,a} \}.
\end{align}
In this case, our $p$-variance function is redefined as 
\begin{align}
    \kappa_p(v,s,a) =& \min_{\lVert P\rVert_p = \beta_{s,a}, \quad  \sum_{s'}P(s')=0, \quad P(s")=0, \quad \forall s"\in F_{s,a}}\langle P,v\rangle\\
    =& \min_{\omega\in\mathbb{R}} \lVert u-\omega\mathbf{1}\rVert_p,\qquad \text{where $u(s) = v(s)\mathbf{1}(s\notin F_{s,a})$.} \\
    =& \kappa_p(u)
\end{align}
This basically says, we consider value of only those states that is allowed (not forbidden) in calculation of $p$-variance. For example, we have
\begin{align}
    \kappa_\infty(v,s,a) & = \frac{\max_{s\notin F_{s,a}}v(s)-\min_{s\notin F_{s,a}}v(s)}{2}.\\
    \end{align}

So theorem 1 of the main paper can be re-stated as 
 \begin{theorem}(Restated) $(\mathtt{Sa})$-rectangular $L_p$ robust Bellman operator is equivalent to reward regularized (non-robust) Bellman operator. That is, using $\kappa_p$ above, we have 
\begin{equation*}\begin{aligned}
    (\mathcal{T}^\pi_{\mathcal{U}^{\mathtt{sa}}_p} v)(s)  =& \sum_{a}\pi(a|s)[  -\alpha_{s,a} -\gamma\beta_{s,a}\kappa_q(v,s,a)  +R_0(s,a) +\gamma \sum_{s'}P_0(s'|s,a)v(s')],\\
    (\mathcal{T}^*_{\mathcal{U}^{\mathtt{sa}}_p} v)(s)  =& \max_{a\in\mathcal{A}}[  -\alpha_{s,a} -\gamma\beta_{s,a}\kappa_q(v,s,a)  +R_0(s,a) +\gamma \sum_{s'}P_0(s'|s,a)v(s')].
\end{aligned}\end{equation*}
\end{theorem}
\subsubsection*{S-Rectangular Uncertainty}
This notion can also be applied to \texttt{s}-rectanular uncertainty, but with little caution. Here, we define forbidden states in state $s$ to be $F_s$  (state dependent) instead of state-action dependent in \texttt{sa}-rectangular case. Here, we define $p$-variance as 
\begin{align}
    \kappa_p(v,s) = \kappa_p(u), \qquad \text{where $u(s) = v(s)\mathbf{1}(s\notin F_s)$.  }
\end{align}
So the theorem 2 can be restated as 
\begin{theorem}(restated) (Policy Evaluation) \texttt{S}-rectangular $L_p$ robust Bellman operator is equivalent to reward regularized (non-robust) Bellman operator, that is 
\begin{equation*}
    (\mathcal{T}^\pi_{\mathcal{U}^s_p} v)(s)  =   -\Bigm(\alpha_s +\gamma\beta_{s}\kappa_q(v,s)\Bigm)\lVert\pi(\cdot|s)\rVert_q  +\sum_{a}\pi(a|s)\Bigm(R_0(s,a) +\gamma \sum_{s'}P_0(s'|s,a)v(s')\Bigm)
\end{equation*}
where $\kappa_p$ is defined above and $\lVert \pi(\cdot|s)\rVert _q$ is $q$-norm of the vector $\pi(\cdot|s)\in\Delta_{\mathcal{A}}$.
\end{theorem}
All the other results (including theorem 4), we just need to replace the old $p$-variance function with new $p$-variance function appropriately.

\section{Application to UCRL}\label{app:UCRL}
In robust MDPs, we consider the minimization over uncertainty set to avoid risk. When we want to discover the underlying kernel by exploration, then we seek optimistic policy, then we consider the maximization over uncertainty set \cite{UCRL,UCRL2,CUCRL}. We refer the reader to the step 3 of the UCRL algorithm \cite{UCRL}, which seeks to find
\begin{align}
    arg\max_{\pi}\max_{R,P \in\mathcal{U}}\langle \mu ,v^{\pi}_{P,R}\rangle,
\end{align}
where \[\mathcal{U} =\{(R,P)\mid \lvert R(s,a)-R_0(s,a)\rvert \leq \alpha_{s,a},\lvert P(s'|s,a)-P_0(s'|s,a)\rvert \leq \beta_{s,a,s'}, P\in(\Delta_\mathcal{S})^{\mathcal{S}\times\mathcal{A}} \}\] for current estimated  kernel $P_0$ and reward function $R_0$. We refer section 3.1.1 and step 4 of the UCRL 2 algorithm of \cite{UCRL2}, which seeks to find 
\begin{align}
    arg\max_{\pi}\max_{R,P \in\mathcal{U}}\langle \mu ,v^{\pi}_{P,R}\rangle,
\end{align}
where \[\mathcal{U} =\{(R,P)\mid \lvert R(s,a)-R_0(s,a)\rvert \leq \alpha_{s,a},\lVert P(\cdot|s,a)-P_0(\cdot|s,a)\rVert_1 \leq \beta_{s,a}, P\in(\Delta_\mathcal{S})^{\mathcal{S}\times\mathcal{A}} \}\] 
The uncertainty radius $\alpha,\beta$ depends on the number of samples of different transitions and observations of the reward. The paper \cite{UCRL} doesn't explain any method to solve the above problem. UCRL 2 algorithm \cite{UCRL2}, suggests to solve it by linear programming that can be very slow. We show that it can be solved by our methods.\\

The above problem can be tackled as 
following
\begin{align}
    \max_{\pi}\max_{R,P\in\mathcal{U}^{sa}_p}\langle \mu ,v^{\pi}_{P,R}\rangle.
\end{align}
We can define, optimistic Bellman operators as  
\begin{align}
    \Hat{\mathcal{T}}^\pi_{\mathcal{U}}v := \max_{R,P\in\mathcal{U}}v^{\pi}_{P,R},\qquad 
    \Hat{\mathcal{T}}^*_{\mathcal{U}}v :=\max_{\pi} \max_{R,P\in\mathcal{U}}v^{\pi}_{P,R}.
\end{align}
The well definition and contraction of the above optimistic operators may follow directly from their pessimistic (robust) counterparts. We can evaluate above optimistic operators as
\begin{align}
    &(\Hat{\mathcal{T}}^\pi_{\mathcal{U}^{sa}_p}v)(s) = \sum_{a}\pi(a|s)\bigm[R_0(s,a) + \alpha_{s,a}+ \beta_{s,a}\gamma\kappa_q(v) + \sum_{s'}P_0(s'|s,a)v(s')\bigm],\\
    &(\Hat{\mathcal{T}}^*_{\mathcal{U}^{sa}_p}v)(s) = \max_{a}\bigm[R_0(s,a) + \alpha_{s,a}+ \beta_{s,a}\gamma\kappa_q(v) + \sum_{s'}P_0(s'|s,a)v(s')\bigm].
\end{align}
The uncertainty radiuses $\alpha,\beta$ and nominal values $P_0,R_0$ can be found by similar analysis by \cite{UCRL,UCRL2}. We can get the Q-learning from the above results as 
\begin{align}
    Q(s,a) \to R_0(s,a) - \alpha_{s,a} -\gamma\beta_{s,a}\kappa_q(v) +\gamma\sum_{s'}P_0(s'|s,a)\max_{a'}Q(s',a'), 
\end{align}
where $v(s) = \max_{a}Q(s,a)$. From law of large numbers, we know that uncertainty radiuses $\alpha_{s,a},\beta_{s,a}$ behaves as $O(\frac{1}{\sqrt{n}})$ asymptotically with number of iteration $n$. This resembles very closely to UCB VI algorithm \cite{UCRLVI}.
We emphasize that similar optimistic operators can be defined and evaluated for s-rectangular uncertainty sets too.

\section{Q-Learning for $(\mathtt{sa})$-rectangular MDPs}\label{app:SALpQL}
In view of theorem \ref{rs:saLprvi}, we can define $Q^\pi_{\mathcal{U}^{\mathtt{sa}}_p}$, the robust Q-values under policy $\pi$ for $(\mathtt{sa})$-rectangular $L_p$ constrained uncertainty set $\mathcal{U}^{\mathtt{sa}}_p$ as
\begin{equation}\begin{aligned}
    &Q^\pi_{\mathcal{U}^{\mathtt{sa}}_p}(s,a) := -\alpha_{s,a} -\gamma\beta_{s,a}\kappa_q(v^\pi_{\mathcal{U}^{\mathtt{sa}}_p})  +R_0(s,a) +\gamma \sum_{s'}P_0(s'|s,a)v^\pi_{\mathcal{U}^{\mathtt{sa}}_p}(s').
\end{aligned}\end{equation}
This implies that we have the following relation between robust Q-values and robust value function, same as its non-robust counterparts,
\begin{equation}
    v^\pi_{\mathcal{U}^{\mathtt{sa}}_p}(s) = \sum_{a}\pi(a|s)Q^\pi_{\mathcal{U}^{\mathtt{sa}}_p}(s,a).
\end{equation}
Let $Q^*_{\mathcal{U}^{\mathtt{sa}}_p}$ denote the optimal robust Q-values associated with optimal robust value $v^*_{\mathcal{U}^{\mathtt{sa}}_p}$, given as
\begin{equation}\begin{aligned}\label{eq:saLpQ}
    &Q^*_{\mathcal{U}^{\mathtt{sa}}_p}(s,a) := -\alpha_{s,a} -\gamma\beta_{s,a}\kappa_q(v^*_{\mathcal{U}^{\mathtt{sa}}_p})  +R_0(s,a) +\gamma \sum_{s'}P_0(s'|s,a)v^*_{\mathcal{U}^{\mathtt{sa}}_p}(s').
\end{aligned}\end{equation}
It is evident from theorem \ref{rs:saLprvi} that optimal robust value and optimal robust Q-values satisfies the following relation, same as its non-robust counterparts,
\begin{equation}\begin{aligned}\label{eq:saLpv}
     v^*_{\mathcal{U}^{\mathtt{sa}}_p}(s') = \max_{a\in\mathcal{A}}Q^*_{\mathcal{U}^{\mathtt{sa}}_p}(s,a).
 \end{aligned}\end{equation}
Combining \ref{eq:saLpv} and \ref{eq:saLpQ}, we have optimal robust Q-value recursion as follows
\begin{equation}\begin{aligned}
  &Q^*_{\mathcal{U}^{\mathtt{sa}}_p}(s,a) = -\alpha_{s,a} -\gamma\beta_{s,a}\kappa_q(v^*_{\mathcal{U}^{\mathtt{sa}}_p})  +R_0(s,a) +\gamma \sum_{s'}P_0(s'|s,a)\max_{a\in\mathcal{A}}Q^*_{\mathcal{U}^{\mathtt{sa}}_p}(s,a).
\end{aligned}\end{equation}
The above robust Q-value recursion enjoys similar properties as its non-robust counterparts. 
\begin{corollary}($(\mathtt{sa})$-rectangular $L_p$ regularized Q-learning) Let
\begin{equation*}\begin{aligned}
    Q_{n+1}(s,a)  =  R_0(s,a)- \alpha_{sa} -\gamma\beta_{sa}\kappa_q(v_n) + \gamma\sum_{s'}P_0(s'|s,a)\max_{a\in\mathcal{A}}Q_n(s',a),
\end{aligned}\end{equation*}
where  $ v_{n}(s) = \max_{a\in\mathcal{A}}Q_{n}(s,a) $, then $Q_n$ converges to $Q^*_{\mathcal{U}^{\mathtt{sa}}_p}$ linearly. 
\end{corollary}
Observe that the above Q-learning equation is exactly the same as non-robust MDP except the reward penalty. Recall that $\kappa_1(v) = 0.5(\max_{s}v(s) - \min_{s}v(s))$ is difference between peak to peak values and $\kappa_2(v)$ is variance of $v$, that can be easily estimated. Hence, model free algorithms for $(\mathtt{sa})$-rectangular $L_p$ robust MDPs for $p=1,2$, can be derived easily from the above results.
This implies that $(\mathtt{sa})$-rectangular $L_1$ and $L_2$ robust MDPs are as easy as non-robust MDPs.

\section{Model Based Algorithms}\label{app:ModelBasedAlgorithms}
In this section, we assume that we know the nominal transitional kernel and nominal reward function. Algorithm \ref{alg:SALp}, algorithm \ref{alg:SLp} is model based algorithm for $(\mathtt{sa})$-rectangular  and $\mathtt{s}$ rectangular $L_p$ robust MDPs respectively. It is explained in the algorithms, how to get deal with specail cases $(p=1,2,\infty)$ in a easy way.

\begin{algorithm}[H]
\caption{Model Based Q-Learning Algorithm for SA Recantangular $L_p$ Robust MDP}\label{alg:SALp}
\begin{algorithmic} [1]
\STATE \textbf{Input}: $\alpha_{s,a},\beta_{s,a}$ are uncertainty radius in reward and transition kernel respectively in state $\mathtt{s}$ and action $a$. Transition kernel $P$ and reward vector $R$. Take initial $Q$-values $Q_0$ randomly and $v_0(s) = \max_{a}Q_0(s,a).$\\
\WHILE{ not converged}
    \STATE Do binary search in $[\min_{s}v_n(s), \max_{s}v_n(s)]$ to get $q$-mean $\omega_n$, such that
    \begin{equation}\label{alg:SLP:eq:kappa}
        \sum_{s}\frac{(v_n(s)-\omega_n)}{|v_n(s)-\omega_n|}|v_n(s) - \omega_n|^{\frac{1}{p-1}} = 0.
    \end{equation}
    \STATE Compute $q$-variance: $\qquad \kappa_n = \lVert v-\omega_n\rVert_q$.
    \STATE Note: For $p=1,2,\infty$, we can compute $\kappa_n$ exactly in closed from, see table \ref{tb:kappa}.
\FOR{$s \in \mathcal{S}$ }
    \FOR{$a \in \mathcal{A}$ }
        \STATE Update Q-value as \[Q_{n+1}(s,a) = R_0(s,a)  - \alpha_{sa} -\gamma\beta_{sa}\kappa_n + \gamma \sum_{s'}P_0(s'|s,a)\max_{a}Q_n(s',a).\]
    \ENDFOR
    \STATE Update value as \[v_{n+1}(s) = \max_{a}Q_{n+1}(s,a).\]
\ENDFOR
    \[n \to n+1\]
\ENDWHILE
\end{algorithmic}
\end{algorithm}

\begin{algorithm}[H]
\caption{Model Based Algorithm for S Recantangular $L_p$ Robust MDP}\label{alg:SLp}
\begin{algorithmic} [1]
\STATE Take initial $Q$-values $Q_0$ and value function $v_0$ randomly. 
\STATE \textbf{Input}: $\alpha_{s},\beta_{s}$ are uncertainty radius in reward and transition kernel respectively in state $\mathtt{s}$.\\
\WHILE{ not converged}
\STATE Do binary search in $[\min_{s}v_n(s), \max_{s}v_n(s)]$ to get $q$-mean $\omega_n$, such that
    \begin{equation}\label{alg:SLP:eq:kappa}
        \sum_{s}\frac{(v_n(s)-\omega_n)}{|v_n(s)-\omega_n|}|v_n(s) - \omega_n|^{\frac{1}{p-1}} = 0.
    \end{equation}
\STATE Compute $q$-variance: $\qquad \kappa_n = \lVert v-\omega_n\rVert_q$.
\STATE Note: For $p=1,2,\infty$, we can compute $\kappa_n$ exactly in closed from, see table \ref{tb:kappa}.
\FOR{$s \in \mathcal{S}$ }
    \FOR{$a \in \mathcal{A}$ }
        \STATE Update Q-value as \begin{equation}\label{alg:SLP:eq:Qupdate}
        Q_{n+1}(s,a) =R_0(s,a) +\gamma\sum_{s'}P_0(s'|s,a) v_{n+1}(s').
        \end{equation}
    \ENDFOR
    \STATE Sort actions in decreasing order of the Q-value, that is
    \begin{equation}\label{alg:SLP:eq:Qsort}
    Q_{n+1}(s,a_i)\geq Q_{n+1}(s,a_{i+1}).
    \end{equation}
    \STATE Value evaluation:
    \begin{equation}\label{alg:SLP:eq:valEval}
    v_{n+1}(s) = x \quad \text{such that }\quad (\alpha_s +\gamma\beta_{s}\kappa_n)^{p} = \sum_{Q_{n+1}(s,a_i)\geq x}|Q_{n+1}(s,a_i) - x|^{p}.
    \end{equation}
    \STATE Note: We can compute $v_{n+1}(s)$ exactly in closed from for $p=\infty$ and for $p=1,2$, we can do the same using algorithm \ref{alg:f1},\ref{alg:f2} respectively, see table \ref{tb:val}.
\ENDFOR
\[n \to n+1\]
\ENDWHILE
\end{algorithmic}
\end{algorithm}

\begin{algorithm}[H]
\caption{Model based algorithm for \texttt{s}-recantangular $L_1$ robust MDPs}\label{alg:Mb:SL1}
\begin{algorithmic} [1]
\STATE Take initial value function $v_0$ randomly and start the counter $n=0$. 
\WHILE{ not converged}
\STATE Calculate $q$-variance: $ \qquad \kappa_n = \frac{1}{2}\bigm[\max_{s}v_n(s) - \min_{s}v_n(s)\bigm]$
\FOR{$s \in \mathcal{S}$ }
    \FOR{$a \in \mathcal{A}$ }
        \STATE Update Q-value as \begin{equation}\label{alg:SLP:eq:Qupdate}
        Q_{n}(s,a) =R_0(s,a) +\gamma\sum_{s'}P_0(s'|s,a) v_{n}(s').
        \end{equation}
    \ENDFOR
    \STATE Sort actions in state $s$, in decreasing order of the Q-value, that is
    \begin{equation}\label{alg:SLP:eq:Qsort}
    Q_{n}(s,a_1)\geq Q_{n}(s,a_{2}),\cdots \geq Q_{n}(s,a_A).
    \end{equation}
    \STATE Value evaluation:
    \begin{equation}\label{alg:SLP:eq:valEval}
    v_{n+1}(s) = \max_{m}\frac{\sum_{i=1}^m Q_{n}(s,a_i) - \alpha_s -\beta_s\gamma\kappa_n}{m}.
    \end{equation}
    \STATE Value evaluation can also be done using algorithm \ref{alg:f1}.
\ENDFOR
\[n \to n+1\]
\ENDWHILE
\end{algorithmic}
\end{algorithm}

\section{Experiments}\label{app:experiments}
The table 4 contains relative cost (time) of robust value iteration w.r.t. non-robust MDP, for randomly generated kernel and reward function with the number of states $S$ and the number of action $A$.

\begin{table}\label{tb:rlt}
  \caption{Relative running cost (time) for value iteration}
  \centering
  \begin{tabular}{llllll}
    \toprule                   
    $\mathcal{U}$     & S=10 A=10 & S=30 A=10 & S=50 A=10 & S=100 A=20     & remark \\
    \midrule
    non-robust & 1 &1 &1 &1 &      \\&\\
    $\mathcal{U}^{sa}_{\infty}$ by LP & 1374 &2282 &2848 &6930 &   scipy.optimize.linearprog    \\&\\
    $\mathcal{U}^{sa}_{1}$ by LP & 1438 &6616 &6622 &16714 &      scipy.optimize.linearprog \\&\\
    $\mathcal{U}^{s}_{1}$ by LP& 72625 &629890 &4904004 &NA &  scipy.optimize.linearprog/minimize     \\&\\
    $\mathcal{U}^{sa}_{1}$& 1.77 &1.38 &1.54 &1.45 & closed form     \\&\\
    $\mathcal{U}^{sa}_{2}$& 1.51 &1.43 &1.91 &1.59 & closed form     \\&\\
    $\mathcal{U}^{sa}_{\infty}$& 1.58 &1.48 &1.37 &1.58 &    closed form  \\&\\
    $\mathcal{U}^{s}_{1}$& 1.41 &1.58 &1.20 &1.16 &closed form      \\&\\
    $\mathcal{U}^{s}_{2}$& 2.63 &2.82 &2.49 &2.18 &     closed form \\&\\
    $\mathcal{U}^{s}_{\infty}$& 1.41 &3.04 &2.25 &1.50 &    closed form  \\&\\
    $\mathcal{U}^{sa}_{5}$& 5.4 &4.91 &4.14 &4.06 &  binary search     \\&\\
    $\mathcal{U}^{sa}_{10}$& 5.56 &5.29 &4.15 &3.26 &binary search      \\&\\
    $\mathcal{U}^{s}_{5}$& 33.30 &89.23 &40.22 &41.22 & binary search     \\&\\
    $\mathcal{U}^{s}_{10}$& 33.59 &78.17 &41.07 &41.10 &    binary search  \\&\\
    \bottomrule
  \end{tabular}
  \end{table}

\subsection*{Notations} 
S : number of state, 
A: number of actions, 
$\mathcal{U}^{sa}_p$ LP:   Sa rectangular $L_p$ robust MPDs by Linear Programming, 
$\mathcal{U}^{s}_p$ LP:  S rectangular $L_p$ robust MPDs by Linear Programming and other numerical methods, 
$\mathcal{U}^{sa/s}_{p=1,2,\infty}$ : Sa/S rectangular $L_1/L_2/L_\infty$ robust MDPs by closed form method (see table 2, theorem 3)
$\mathcal{U}^{sa/s}_{p=5,10}$ : Sa/S rectangular $L_5/L_{10}$ robust MDPs by binary search (see table 2, theorem 3 of the paper)

\subsection*{Observations}
1. Our method for s/sa rectangular $L_1/L_2/L_\infty$ robust MDPs takes almost same (1-3 times) the time as non-robust MDP for one iteration of value iteration. This confirms our complexity analysis (see table 4 of the paper)
2.  Our binary search method for sa rectangular $L_5/L_{10}$ robust MDPs takes around $4-6$ times more time than non-robust counterpart. This is due to extra iterations required to find p-variance function $\kappa_p(v)$ through binary search.
3.  Our binary search method for s rectangular $L_5/L_{10}$ robust MDPs takes around $30-100$ times more time than non-robust counterpart. This is due to extra iterations required to find p-variance function $\kappa_p(v)$ through binary search and Bellman operator.
4.  One common feature of our method is that time complexity scales moderately as guranteed through our complexity analysis.
5.  Linear programming methods for sa-rectangualr $L_1/L_\infty$ robsust MDPs take atleast 1000 times more than our methods for small state-action space, and it scales up very fast.
6. Numerical methods (Linear programming for minimization over uncertianty and 'scipy.optimize.minimize' for maximization over policy) for s-rectangular $L_1$ robust MDPs take 4-5 order more time than our mehtods (and non-robust MDPs) for very small state-action space, and scales up too fast. The reason is obvious, as it has to solve two optimization, one minimization over uncertainty and other maximization over policy, whereas in the sa-rectangular  case, only minimization over uncertainty is required.  This confirms that s-rectangular uncertainty set is much more challenging.

\subsection*{Rate of convergence}
The rate of convergence for all were approximately the same as $0.9 = \gamma$, as predicted by theory. And it is well illustrated by the relative rate of convergence w.r.t. non-robust by the table \ref{tb:roc}.

\begin{table}\label{tb:roc}
  \caption{Relative running cost (time) for value iteration}
  \centering
  \begin{tabular}{llll}
    \toprule                   
    $\mathcal{U}$     & S=10 A=10 &  S=100 A=20     & remark \\
    \midrule
    non-robust & 1 &1 &      \\&\\
    $\mathcal{U}^{sa}_{1}$& 0.999 &0.999 & closed form     \\&\\
    $\mathcal{U}^{sa}_{2}$&0.999 &0.999 & closed form     \\&\\
    $\mathcal{U}^{sa}_{\infty}$& 1.000 &0.998  &    closed form  \\&\\
    $\mathcal{U}^{s}_{1}$& 0.999 &0.999 &
    closed form      \\&\\
    $\mathcal{U}^{s}_{2}$& 0.999 &0.999 &    closed form \\&\\
    $\mathcal{U}^{s}_{\infty}$& 1.000 &0.998 &    closed form  \\&\\
    $\mathcal{U}^{sa}_{5}$& 0.999 &0.995 & binary search     \\&\\
    $\mathcal{U}^{sa}_{10}$& 1.000 &0.999  &binary search      \\&\\
    $\mathcal{U}^{s}_{5}$& 1.000 &0.999  & binary search     \\&\\
    $\mathcal{U}^{s}_{10}$& 1.000 &0.995 &   binary search  \\&\\
    \bottomrule
  \end{tabular}
  \end{table}
 
In the above experiments, Bellman updates for sa/s rectangular $L_1/L_2/L_\infty$ were done in closed form, and for $L_5/L_{10}$ were done by binary search as suggested by table 2 and theorem 3. 

Note: Above experiments' results are for few runs, hence containing some stochasticity but the general trend is clear. In the final version, we will do averaging of many runs to minimize the stochastic nature. Results for many different runs can be found at https://github.com/******.

Note that the above experiments were done without using too much parallelization. There is ample scope to fine-tune and improve the performance of robust MDPs. The above experiments confirm the theoretical complexity provided in Table 4 of the paper. The codes and results can be found at https://github.com/******.

\subsection*{Experiments parameters}  Number of states $S$ (variable), number of actions $A$ (variable), transition kernel and reward function generated randomly,  discount factor $0.9$,  uncertainty radiuses =$0.1$ (for all states and action, just for convenience ),  number of iterations = 100,  tolerance for binary search =  $0.00001$
\subsection*{Hardware} The experiments are done on the following hardware: Intel(R) Core(TM) i5-4300U CPU @ 1.90GHz 64 bits, memory 7862MiB
Software: Experiments were done in python, using numpy, scipy.optimize.linprog for Linear programmig for policy evalution in s/sa rectangular robust MDPs, scipy.optimize.minize and scipy.optimize.LinearConstraints for policy improvement in s-rectangular $L_1$ robust MDPs.

\section{Extension to Model Free Settings}
Extension of Q-learning (in section \ref{app:SALpQL} ) for \texttt{sa}-rectangular MDPs to model free setting can easily done similar to \cite{Rcontamination}, also policy gradient method can be obtained as \cite{PG_RContamination}. The only thing, we need to do, is to be able to compute/estimate $\kappa_q$ online. It can be estimated using an ensemble (samples). Further,   $\kappa_2$ can be estimated by the estimated mean and the estimated second moment. $\kappa_\infty$ can be estimated by tracking maximum and minimum values.\\

For \texttt{s}-rectangular case too, we can obtain model-free algorithms easily, by estimating $\kappa_q$ online and keeping track of Q-values and value function. The convergence analysis may be similar to \cite{Rcontamination}, especially for \texttt{sa}-rectangular case, and for the other, it would be two time scale, which can be dealt with techniques in \cite{borkarBook}. We leave this for future work. It would be interesting to obtain policy gradient methods for this case, which we believe can be obtained from the policy evaluation theorem.

\section{$p$-variance}
\label{app:pvarianceSection}
Recall that $\kappa_p$ is defined as follows 
\[\kappa_p(v) =\min_{w} \lVert v-\omega\mathbf{1}\rVert_p = \lVert v-\omega_p\rVert_p.\]
Now, observe that
\begin{equation}\begin{aligned}
&\frac{\partial \lVert v -\omega\rVert_p}{\partial \omega}  = 0\\
\implies &  \sum_{s}sign(v(s)-\omega)|v(s) - \omega|^{p-1} = 0,\\
\implies &  \sum_{s}sign(v(s)-\omega_p(v))|v(s) - \omega_q(v)|^{p-1} = 0.
\end{aligned}
\end{equation}
For $p=\infty$,  we have
\begin{equation}\begin{aligned}
&\lim_{p\to\infty}\Bigm\lvert\sum_{s}sign\bigm(v(s)-\omega_\infty(v)\bigm)\bigm\lvert v(s) - \omega_\infty(v)\bigm\rvert^{p}\Bigm\rvert^\frac{1}{p} = 0 \\
\implies&\bigm(\max_{s}\lvert v(s) - \omega_\infty(v)\rvert\bigm)\lim_{p\to\infty}\Bigm\lvert\sum_{s}sign\bigm(v(s)-\omega_\infty(v)\bigm)\Bigm(\frac{\lvert v(s) - \omega_\infty(v)\bigm\rvert}{\max_{s}\lvert v(s) - \omega_\infty(v)\rvert}\Bigm)^{p}\Bigm\rvert^\frac{1}{p} = 0 \\
&\text{Assuming $\max_{s}\lvert v(s) - \omega_\infty(v)\rvert \neq 0$ otherwise $\omega_\infty = v(s)=v(s'),\quad \forall s,s'$ }\\
\implies&\lim_{p\to\infty}\Bigm\lvert\sum_{s}sign\bigm(v(s)-\omega_\infty(v)\bigm)\Bigm(\frac{\lvert v(s) - \omega_\infty(v)\bigm\rvert}{\max_{s}\lvert v(s) - \omega_\infty(v)\rvert}\Bigm)^{p}\Bigm\rvert^\frac{1}{p} = 0 \\
&\text{To avoid technical complication, we assume $\max_{s}v(s)>v(s)< \min_{s}v(s), \quad \forall s$}\\
\implies&\lim_{p\to\infty} \lvert\max_{s} v(s) - \omega_\infty(v)\rvert =\lim_{p\to\infty}\lvert\min_{s} v(s) - \omega_\infty(v)\rvert\\
\implies& \max_{s} v(s) - \lim_{q\to\infty}\omega_\infty(v) =-(\min_{s} v(s) - \lim_{p\to\infty}\omega_\infty(v)),\qquad \text{(managing signs)}\\
\implies&\lim_{p\to\infty}\omega_\infty(v) = \frac{\max_{s}v(s) +\min_{s}v(s)}{2}.
\end{aligned}\end{equation}

\begin{equation}\begin{aligned}
    \kappa_\infty(v) =& \lVert v-\omega_{\infty}\mathbf{1}\rVert_\infty\\
    =& \lVert v-\frac{\max_{s}v(s) + \min_{s}v(s)}{2}\mathbf{1}\rVert_\infty, \qquad \text{(putting in value of $\omega_\infty$)}\\
    =&\frac{\max_{s}v(s) - \min_{s}v(s)}{2}
\end{aligned}\end{equation}

For $p=2$, we have
\begin{equation}\begin{aligned}
    \kappa_2(v) =& \lVert v-\omega_{2}\mathbf{1}\rVert_2\\
    =& \lVert v-\frac{\sum_{s}v(s)}{S}\mathbf{1}\rVert_2,\\
    =&\sqrt{\sum_{s}(v(s) -\frac{\sum_{s}v(s)}{S})^2}
\end{aligned}\end{equation}

For $p=1$, we have
\begin{equation}\begin{aligned}
&\sum_{s\in\mathcal{S}} sign\bigm(v(s) - \omega_1(v)\bigm) = 0 \\
\end{aligned}\end{equation}
Note that there may be more than one values of $\omega_1(v)$ that satisfies the above equation and each solution does equally good job (as we will see later). So we will pick one ( is median of $v$) according to our convenience as
\[ \omega_1(v) = \frac{v(s_{\lfloor (S+1)/2\rfloor}) +v(s_{\lceil (S+1)/2\rceil})}{2} \quad \text{where} \quad v(s_i)\geq v(s_{i+1}) \quad \forall i.\] 

\begin{equation}\begin{aligned}
    \kappa_1(v) =& \lVert v-\omega_{1}\mathbf{1}\rVert_1\\
    =& \lVert v-med(v)\mathbf{1}\rVert_1, \qquad \text{(putting in value of $\omega_0$, see table \ref{tb:mean})}\\
    =&\sum_{s}\lvert v(s)-med(v)\rvert\\
    =&\sum_{i=1}^{\lfloor (S+1)/2\rfloor} (v(s)-med(v))  +\sum_{\lceil (S+1)/2\rceil}^{S} (med(v) -v(s)) \\
    =&\sum_{i=1}^{\lfloor (S+1)/2\rfloor} v(s)  -\sum_{\lceil (S+1)/2\rceil}^{S} v(s) 
\end{aligned}\end{equation}
where 
$med(v) := \frac{v(s_{\lfloor (S+1)/2\rfloor}) +v(s_{\lceil (S+1)/2\rceil})}{2} \quad \text{where} \quad v(s_i)\geq v(s_{i+1}) \quad \forall i$ is median of $v$. The results are summarized in table \ref{tb:kappa} and \ref{tb:mean}.
\begin{table}
  \caption{$p$-mean, where $v(s_i)\geq v(s_{i+1}) \quad \forall i.$}
  \label{tb:mean}
  \centering
  \begin{tabular}{lll}
    \toprule                   
    $x$     & $\omega_x(v)$     & remark \\
    \midrule
    $p$ & $\sum_{s}sign(v(s)-\omega_p(v))\lvert v(s) - \omega_p(v)\rvert^{\frac{1}{p-1}} = 0 $ & Solve by binary search\\\\
    $1$     & $\frac{v(s_{\lfloor (S+1)/2\rfloor}) +v(s_{\lceil (S+1)/2\rceil})}{2} $      & Median  \\\\
    2     &  $\frac{\sum_{s}v(s)}{S}    $ & Mean\\\\
    $\infty$     & $\frac{\max_{s}v(s) + \min_{s}v(s)}{2}$      &  Average of peaks  \\
    \bottomrule
  \end{tabular}
\end{table}


\subsection{$p$-variance function $\kappa_p$ and kernel noise}

\begin{lemma}
\label{regfn}
$q$-variance function $\kappa_q$ is the solution of the following optimization problem (kernel noise),
\[\kappa_q(v) = -\frac{1}{\epsilon}\min_{c}\langle c,v\rangle, \qquad \lVert c\rVert_p\leq \epsilon, \qquad \sum_{s}c(s) = 0.\]
\end{lemma}
\begin{proof}
Writing Lagrangian $L$, as
\[L := \sum_{s}c(s)v(s) + \lambda\sum_{s}c(s) + \mu(\sum_{s}\lvert c(s)\rvert^p -\epsilon^p),\]
where $\lambda \in\mathbb{R}$ is the multiplier for the constraint $\sum_{s}c(s) = 0$ and $\mu \geq 0$ is the multiplier for the inequality constraint $\lVert c\lVert_q \leq \epsilon.$ Taking its derivative, we have
\begin{equation}\begin{aligned}
    \frac{\partial L}{\partial c(s)} = v(s) + \lambda + \mu p \lvert c(s)\rvert^{p-1}\frac{c(s)}{\lvert c(s)\rvert}
\end{aligned}\end{equation}
From the KKT (stationarity)  condition, the solution $c^
*$ has zero derivative, that is 
\begin{equation}\begin{aligned}\label{app:eq:kappa:LagDer}
     v(s) + \lambda + \mu p \lvert c^*(s)\rvert^{p-1}\frac{c^*(s)}{\lvert c^*(s)\rvert} = 0, \qquad \forall s\in\mathcal{S}.
\end{aligned}\end{equation}

Using Lagrangian derivative equation \eqref{app:eq:kappa:LagDer}, we have
\begin{equation}\begin{aligned}
    & v(s) + \lambda + \mu p \lvert c^*(s)\rvert^{p-1}\frac{c^*(s)}{\lvert c^*(s)\rvert} = 0 \\
    \implies & \sum_{s}c^*(s)[v(s)  + \lambda + \mu p \lvert c^*(s)\rvert^{p-1}\frac{c^*(s)}{\lvert c^*(s)\rvert}] = 0, \qquad \text{(multiply with $c^*(s)$ and summing )}\\
     \implies & \sum_{s}c^*(s)v(s)  + \lambda\sum_{s}c^*(s) + \mu p \sum_{s}\lvert c^*(s)\rvert^{p-1}\frac{(c^*(s))^2}{\lvert c^*(s)\rvert} = 0\\
    \implies & \langle c^*,v\rangle + \mu p \sum_{s}\lvert c^*(s)\rvert^{p} = 0 \qquad \text{(using $\sum_{s}c^*(s) =0$ and $(c^*(s))^2 = \lvert c^*(s)\rvert^2$ )}\\
    \implies & \langle c^*,v\rangle = -\mu p \epsilon^p, \qquad \text{(using $\sum_{s}\lvert c^*(s)\rvert^p= \epsilon^p$ ).} 
\end{aligned}\end{equation}
It is easy to see that $\mu \geq 0$, as minimum value of the objective must not be positive ( at $c=0$, the objective value is zero).
Again we use Lagrangian derivative \eqref{app:eq:kappa:LagDer} and try to get the objective value  ($-\mu p \epsilon^p$) in terms of $\lambda$, as 
\begin{equation}\begin{aligned}
    &v(s) + \lambda + \mu p \lvert c^*(s)\rvert^{p-1}\frac{c^*(s)}{\lvert c^*(s)\rvert} = 0\\
    \implies &\lvert c^*(s)\rvert^{p-2}c^*(s) =-\frac{v(s) + \lambda}{\mu p}, \qquad \text{(re-arranging terms)} \\
    \implies &\sum_{s}|(\lvert c^*(s)\rvert^{p-2}c^*(s))|^{\frac{p}{p-1}} =\sum_{s}|-\frac{v(s) + \lambda}{\mu p}|^{\frac{p}{p-1}}, \qquad \text{(doing $\sum_s\lvert\cdot\rvert^\frac{p}{p-1}$ )} \\
    \implies & \lVert c^*\rVert^p_p =\sum_{s}|-\frac{v(s) + \lambda}{\mu p}|^{\frac{p}{p-1}} = \sum_{s}|\frac{v(s) + \lambda}{\mu p}|^{q} =\frac{\lVert v + \lambda\rVert^q_q}{|\mu p|^q} \\
    \implies & |\mu p|^q\lVert c^*\rVert^p_p = \lVert v + \lambda\rVert^q_q , \qquad \text{(re-arranging terms)}\\
    \implies & |\mu p|^q \epsilon^p=\lVert v + \lambda\rVert^q_q, \qquad \text{(using $\sum_{s}\lvert c^*(s)\rvert^p= \epsilon^p$ )}\\
    \implies & \epsilon (\mu p \epsilon^{p/q})=  \epsilon \lVert v + \lambda\rVert_q\qquad \text{(taking $\frac{1}{q}$the power then multiplying with $\epsilon$)} \\
    \implies &\mu p \epsilon^p=\epsilon \lVert v + \lambda\rVert_q.
\end{aligned}\end{equation}

Again, using Lagrangian derivative \eqref{app:eq:kappa:LagDer} to solve for $\lambda$, we have
\begin{equation}\begin{aligned}
    &v(s) + \lambda + \mu p \lvert c^*(s)\rvert^{p-1}\frac{c^*(s)}{\lvert c^*(s)\rvert} = 0\\
    \implies &\lvert c^*(s)\rvert^{p-2}c^*(s) =-\frac{v(s) + \lambda}{\mu p} , \qquad \text{(re-arranging terms)}\\
    \implies & \lvert c^*(s)\rvert = |\frac{v(s) + \lambda}{\mu p}|^{\frac{1}{p-1}},  \quad \text{(looking at absolute value)}\\
    \qquad\qquad &\quad \text{and}\quad \frac{c^*(s)}{\lvert c^*(s)\rvert} = -\frac{v(s)+\lambda}{|v(s)+\lambda|}, \quad \text{(looking at sign: and note $\mu,p\geq 0$)}\\
    \implies &  \sum_{s}\frac{c^*(s)}{\lvert c^*(s)\rvert}\lvert c^*(s)\rvert = -\sum_{s}\frac{v(s)+\lambda}{|v(s)+\lambda|}|\frac{v(s) + \lambda}{\mu p}|^{\frac{1}{p-1}}, \qquad \text{(putting back)}\\
    \implies &  \sum_{s}c^*(s) = -\sum_{s}\frac{v(s)+\lambda}{|v(s)+\lambda|}|\frac{v(s) + \lambda}{\mu p}|^{\frac{1}{p-1}}, \\
    \implies & \sum_{s}\frac{v(s)+\lambda}{|v(s)+\lambda|}|v(s) + \lambda|^{\frac{1}{p-1}} = 0, \qquad \text{( using $\sum_i c^*(s) = 0$)}
\end{aligned}\end{equation}
Combining everything, we have 
\begin{equation}\label{kappa}\begin{aligned}
    &-\frac{1}{\epsilon}\min_{c}\langle c,v\rangle, \qquad \lVert c\rVert_p\leq \epsilon, \qquad \sum_{s}c(s) = 0\\
    = &\lVert v -\lambda\rVert_q, \quad\text{such that}\quad \sum_{s}sign(v(s)-\lambda)|v(s) - \lambda|^{\frac{1}{p-1}} = 0.
\end{aligned}\end{equation}
Now, observe that
\begin{equation}\begin{aligned}
&\frac{\partial \lVert v -\lambda\rVert_q}{\partial \lambda}  = 0\\
\implies &  \sum_{s}sign(v(s)-\lambda)|v(s) - \lambda|^{\frac{1}{p-1}} = 0,\\
\implies & \kappa_q(v) = \lVert v -\lambda\rVert_q, \quad\text{such that}\quad \sum_{s}sign(v(s)-\lambda)|v(s) - \lambda|^{\frac{1}{p-1}} = 0.
\end{aligned}
\end{equation}
The last equality follows from the convexity of p-norm $\lVert \cdot \rVert_q$, where every local minima is global minima. 


For the sanity check, we re-derive things for $p=1$ from scratch.
For $p=1$, we have
\begin{equation}\begin{aligned}
&-\frac{1}{\epsilon}\min_{c}\langle c,v\rangle, \qquad \lVert c\rVert_1\leq \epsilon, \qquad \sum_{s}c(s) = 0.\\
=& - \frac{1}{2}(\min_{s}v(s) - \max_{s}v(s))\\
=& \kappa_1(v).   
\end{aligned}\end{equation}
It is easy to see the above result, just by inspection.
\end{proof}

\subsection{Binary search for $\omega_p$ and estimation of $\kappa_p$}\label{app:BSkappa}
If the function $f:[-B/2,B/2] \to \mathbb{R}, B\in\mathbb{R}$ is monotonic (WLOG let it be monotonically decreasing) in a bounded domain, and it has a unique root $x^*$ s.t. $f(x^*) = 0$. Then we can find $x$ that is an $\epsilon$-approximation $x^*$ (i.e. $\lVert x - x^*\rVert \leq \epsilon$ ) in $O(B/\epsilon)$ iterations. Why?  Let $x_0=0$ and
\[x_{n+1} := \begin{cases} \frac{-B + x_n}{2} & \text{if}\quad f(x_n) > 0\\
\frac{B + x_n}{2} & \text{if}\quad f(x_n) < 0\\
 x_n & \text{if}\quad f(x_n) = 0\\\end{cases}.\]
It is easy to observe that $\lVert x_n - x^*\rVert \leq B(1/2)^n$. This is proves the above claim. This observation will be referred to many times.\\
Now, we move to the main claims of the section.

\begin{proposition}\label{prp:SLpValMeanEval}
The function   \[h_p(\lambda):=\sum_{s}sign\bigm(v(s)-\lambda\bigm)\bigm\lvert v(s) - \lambda\bigm\rvert^{p}\] is monotonically strictly decreasing and also has a root in the range $[\min_{s}v(s),\max_{s}v(s)]$.
\end{proposition}
\begin{proof}
\begin{equation}\begin{aligned}
    &h_p(\lambda) = \sum_{s}\frac{v(s)-\lambda}{|v(s) - \lambda|}|v(s) - \lambda|^{p}\\
    &\frac{dh_p}{d\lambda}(\lambda) = -p\sum_{s}|v(s) - \lambda|^{p-1} \leq 0,\qquad \forall p\geq 0.
\end{aligned}\end{equation}
Now, observe that $h_p(\max_{s}v(s)) \leq 0$ and $h_p(\min_{s}v(s)) \geq 0$, hence by $h_p$ must have a root in the range $[\min_{s}v(s),\max_{s}v(s)]$ as the function is continuous.
\end{proof}
The above proposition ensures that a root $\omega_p(v)$ can be easily found by binary search between $[\min_{s}v(s),\max_{s}v(s)]$. Precisely, $\epsilon$ approximation of $\omega_p(v)$ can be found in $O(\log(\frac{\max_{s}v(s)-\min_{s}v(s)}{\epsilon}))$  number of iterations of binary search. And one evaluation of the function $h_p$ requires $O(S)$ iterations. And we have finite state-action space and bounded reward hence WLOG we can assume $\lvert\max_{s}v(s)\rvert, \lvert\min_{s}v(s)\rvert$ are bounded by a constant.  Hence, the complexity to approximate $\omega_p$ is $O(S\log(\frac{1}{\epsilon}))$.\\

 Let  $\hat{\omega}_{p}(v)$ be an $\epsilon$-approximation  of $\omega_p(v)$, that is 
\[\bigm\lvert \omega_p(v) - \hat{\omega}_p(v) \bigm\rvert \leq \epsilon.\]

And let $\hat{\kappa}_p(v)$ be approximation of $\kappa_p(v)$ using approximated mean, that is, 
\[\hat{\kappa}_p(v) := \lVert v -\hat{\omega}_{p}(v)\mathbf{1}\rVert_p.\]
Now we will show that $\epsilon$ error in calculation of $p$-mean $\omega_p$, induces $O(\epsilon)$ error in estimation of $p$-variance $\kappa_p$. Precisely,
\begin{equation}\begin{aligned}
    \Bigm\lvert \kappa_p(v)-\hat{\kappa}_p(v)\Bigm\rvert =& \Bigm\lvert \bigm\lVert v -\omega_{p}(v)\mathbf{1}\bigm\rVert_p -\bigm\lVert v -\hat{\omega}_{p}(v)\mathbf{1}\bigm\rVert_p\Bigm\rvert \\
    \leq& \bigm\lVert \omega_{p}(v)\mathbf{1}-\hat{\omega}_{p}(v)\mathbf{1}\bigm\rVert_p,\qquad \text{(reverse triangle inequality)}\\ 
    =&  \bigm\lVert\mathbf{1}\bigm\rVert_p\bigm\lvert \omega_{p}(v)-\hat{\omega}_{p}(v)\bigm\rvert\\
    \leq&  \bigm\lVert\mathbf{1}\bigm\rVert_p\epsilon\\
    = &S^\frac{1}{p}\epsilon \leq S\epsilon. 
\end{aligned}\end{equation}
For general $p$, an $\epsilon$ approximation of $\kappa_p(v)$ can be calculated in 
$O(S\log(\frac{S}{\epsilon})$ iterations. Why? We will estimate mean $\omega_p$ to an $\epsilon/S$ tolerance (with cost $O(S\log(\frac{S}{\epsilon})$ ) and then approximate the $\kappa_p$ with this approximated mean (cost $O(S)$).

\section{ $L_p$ Water Pouring  lemma}\label{app:waterPouringSection}
In this section, we are going to discuss the following optimization problem,
\[\max_{c}-\alpha \lVert c\rVert_q + \langle c,b\rangle \qquad \text{such that }\qquad \sum_{i=1}^Ac_i = 1,\quad c_i\geq 0,\quad \forall i\]
where $\alpha\geq 0$, referred as $L_p$-water pouring problem. We are going to assume WLOG that $b$ is sorted component wise, that is $b_1\geq b_2,\cdots\geq b_A.$ The above problem for $p=2$, is studied in \cite{anava2016k}. The approach we are going to solve the problem is as follows: a) Write Lagrangian b) Since the problem is convex, any solutions of KKT condition is global maximum. c) Obtain conditions using KKT conditions.

\begin{lemma}\label{LpwaterPouring}
Let $b \in\mathbb{R}^A$ be such that its components are in decreasing order (i,e $b_{i}\geq b_{i+1}$), $\alpha\geq 0$ be any non-negative constant, and 
\begin{equation}\label{app:eq:LpwaterPouringlemma}
 \zeta_p := \max_{c}-\alpha \lVert c\rVert_q + \langle c,b\rangle \qquad \text{such that }\qquad \sum_{i=1}^Ac_i = 1,\quad c_i\geq 0,\quad \forall i,    
\end{equation}
and let $c^*$ be a solution to the above problem.
 Then
\begin{enumerate}
    \item \label{app:wp:order} Higher components of $b$, gets higher weight in $c^*$. In other words, $c^*$ is also sorted component wise in descending order, that is \[c^*_1 \geq c^*_2,\cdots, \geq c^*_A.\] 
    \item \label{app:wp:zeta}The value $\zeta_p$ satisfies the following equation
    \[\alpha^{p} = \sum_{ b_i\geq \zeta_p}(b_i - \zeta_p)^{p}\]
    \item \label{app:wpl:policy} The solution $c$ of \eqref{app:eq:LpwaterPouringlemma}, is related to $\zeta_p$ as
    \begin{equation*}
        c_i = \frac{(b_i - \zeta_p)^{p-1}\mathbf{1}(b_i\geq\zeta_p)}{\sum_{s}(b_i - \zeta_p)^{p-1}\mathbf{1}(b_i\geq \zeta_p)}
    \end{equation*} 
    \item \label{app:wp:chi}Observe that the top $\chi_p:=\max\{i|b_i\geq\zeta_p\} $ actions are active and rest are passive. The number of active actions can be calculated as 
\[\{ k | \alpha^{p} \geq \sum_{i=1}^{k}(b_i - b_k)^{p}\} = \{1,2,\cdots,\chi_p\}.\]
\item Things can be re-written as
\[c_i  \propto \begin{cases}(b_i - \zeta_p)^{p-1} & \text{if}\quad i\leq \chi_p\\0 & \text{else}
\end{cases}\qquad \text{and}\qquad \alpha^{p} = \sum_{i=1}^{ \chi_p}(b_i - \zeta_p)^{p}\]

\item \label{app:wp:zetaBinarySearch}The function $\sum_{b_i \geq x}(b_i - x)^p $ is monotonically decreasing in $x$, hence the root $\zeta_p$ can be calculated efficiently by binary search between $[b_1 -\alpha, b_1]$.
    
\item \label{app:wp:zetaBound} Solution is sandwiched as follows
\[b_{\chi_p+1} \leq \zeta_p \leq b_{\chi_p}\]
\item \label{app:wp:chiSol}  $k\leq \chi_p$ if and only if there exist the solution of the following,
\[ \sum_{i=1}^k(b_i-x)^p = \alpha^p \quad \text{and}\quad x \leq b_k.\]

\item \label{app:wp:greedyInclusion} If action $k$ is active  and there is greedy increment hope then action $k+1$ is also active. That is 
    \[k \leq \chi_p \quad\text{and}\quad\lambda_k \leq b_{k+1} \implies k+1\leq \chi_p,\]
    where 
    \[ \sum_{i=1}^k(b_i-\lambda_k)^p = \alpha^p \quad \text{and}\quad \lambda_k \leq b_k.\]
\item \label{app:wp:stoppingCondition} If action $k$ is active, and there is no greedy hope and then action $k+1$ is not active. That is,\[k \leq \chi_p \quad\text{and}\quad\lambda_k > b_{k+1} \implies k+1> \chi_p,\]
    where 
    \[ \sum_{i=1}^k(b_i-\lambda_k)^p = \alpha^p \quad \text{and}\quad \lambda_k \leq b_k.\]
And this implies $k = \chi_p.$
\end{enumerate}
\end{lemma}
\begin{proof}
 \begin{enumerate}
    \item  Let \[f(c): = -\alpha\lVert c\rVert_q + \langle b,c\rangle.\]
    Let $c$ be any vector, and $c'$ be rearrangement $c$ in descending order. Precisely,
    \[c'_{k} := c_{i_k} , \quad \text{where}\quad c_{i_1}\geq c_{i_2},\cdots,\geq c_{i_A}.\]
    Then it is easy to see that $f(c')\geq f(c).$ And the claim follows.
    
     \item Writting Lagrangian of the optimization problem, and its derivative,
\begin{equation}\begin{aligned}
    &L = -\alpha \lVert c\rVert_q + \langle c,b\rangle + \lambda(\sum_{i}c_i -1) + \theta_ic_i\\
    &\frac{\partial L}{\partial c_i} = -\alpha \lVert c\rVert_q^{1-q}|c_i|^{q-2}c_i + b_i + \lambda + \theta_i,
\end{aligned}\end{equation}
$\lambda\in\mathbb{R}$ is multiplier for equality constraint $\sum_{i}c_i = 1$ and $\theta_1,\cdots,\theta_A \geq 0$ are multipliers for inequality constraints $c_i\geq 0,\quad \forall i\in [A].$ Using KKT (stationarity) condition, we have
\begin{equation}\label{app:eq:Lpwp:st}
     -\alpha \lVert c^*\rVert_q^{1-q}|c^*_i|^{q-2}c^*_i + b_i + \lambda + \theta_i = 0\\
\end{equation}
Let $\mathcal{B}:=\{i | c^*_i > 0\}$, then
\begin{equation}\begin{aligned}
    &\sum_{i\in\mathcal{B}} c^*_i[-\alpha \lVert c^*\rVert_q^{1-q}|c^*_i|^{q-2}c^*_i + b_i + \lambda ] = 0\\
    \implies &-\alpha \lVert c^*\rVert_q^{1-q}\lVert c^*\rVert^q_q + \langle c^*, b \rangle + \lambda  = 0, \qquad \text{(using $\sum_i c^*_i = 1$ and $(c^*_i)^2 = |c^*_i|^2$)}\\
     \implies &-\alpha \lVert c^*\rVert_q + \langle c^*, b \rangle + \lambda  = 0\\
     \implies &-\alpha \lVert c^*\rVert_q + \langle c^*, b \rangle  =- \lambda, \qquad \text{(re-arranging)} 
\end{aligned}\end{equation}

Now again using \eqref{app:eq:Lpwp:st},  we have 
\begin{equation}\begin{aligned}
   &-\alpha \lVert c^*\rVert_q^{1-q}|c^*_i|^{q-2}c^*_i + b_i + \lambda + \theta_i = 0\\
    \implies& \alpha\lVert c^*\rVert_q^{1-q}|c^*_i|^{q-2}c^*_i =  b_i + \lambda +\theta_i , \qquad \forall i, \qquad \text{(re-arranging)}
 \end{aligned}\end{equation}
 Now, if $i\in\mathcal{B}$ then $\theta_i = 0$ from complimentry slackness, so we have 
 \[\alpha\lVert c^*\rVert_q^{1-q}|c^*_i|^{q-2}c^*_i =  b_i + \lambda > 0 , \qquad \forall i\in\mathcal{B}\]
 by definition of $\mathcal{B}$. Now, if for some $i$, $b_i+\lambda > 0$ then $b_i+\lambda + \theta_i > 0$ as $\theta_i\geq 0$, that implies
 \[\alpha\lVert c^*\rVert_q^{1-q}|c^*_i|^{q-2}c^*_i =  b_i + \lambda +\theta_i > 0 \]
 \[\implies c^*_i > 0 \implies i\in\mathcal{B}. \]
 So, we have, 
 \[i\in\mathcal{B} \iff b_i+\lambda > 0.\]
  To summarize, we have 
\begin{equation}\label{app:eq:wp:c}
    \alpha\lVert c^*\rVert_q^{1-q}|c^*_i|^{q-2}c^*_i =  (b_i + \lambda) \mathbf{1}(b_i\geq-\lambda), \quad \forall i,\\
\end{equation}    
\begin{equation}\begin{aligned}
     \implies & \sum_{i}\alpha^{\frac{q}{q-1}}\lVert c^*\rVert_q^{-q}(c^*_i)^q = \sum_{i}(b_i + \lambda)^{\frac{q}{q-1}}\mathbf{1}(b_i\geq-\lambda), \qquad \text{(taking $q/(q-1)$th power and summming)} \\
     \implies & \alpha^{p}  = \sum_{i=1}^A(b_i + \lambda)^{p}\mathbf{1}(b_i\geq-\lambda).
\end{aligned}\end{equation}
So, we have,
\begin{equation}\begin{aligned}\label{eq:waterpourng}
    \zeta_p    &= -\lambda \quad \text{such that }\quad \alpha^{p} = \sum_{b_i\geq \lambda}(b_i + \lambda)^{p}.\\
   \implies \alpha^{p} &= \sum_{ b_i\geq \zeta_p}(b_i - \zeta_p)^{p}
\end{aligned}\end{equation}
\item Furthermore, using \eqref{app:eq:wp:c}, we have
\begin{equation}\begin{aligned}
   &\alpha\lVert c^*\rVert_q^{1-q}|c^*_i|^{q-2}c^*_i =  (b_i + \lambda) \mathbf{1}(b_i\geq -\lambda) = (b_i -\zeta_p) \mathbf{1}(b_i\geq \zeta_p) \quad \forall i,\\
   \implies &  c^*_i \propto (b_i -\zeta_p)^{\frac{1}{q-1}}\mathbf{1}(b_i\geq \zeta_p)= \frac{(b_i - \zeta_p)^{p-1}\mathbf{1}(b_i\geq\zeta_p)}{\sum_{i}(b_i - \zeta_p)^{p-1}\mathbf{1}(b_i\geq \zeta_p)}, \qquad \text{(using $\sum_{i}c^*_i = 1$)}. \\
\end{aligned}\end{equation}

\item Now, we move on to calculate the number of active actions $\chi_p$. Observe that the function
\begin{equation}
    f(\lambda) := \sum_{i=1}^A(b_i-\lambda)^p\mathbf{1}(b_i\geq \lambda) - \alpha^p
\end{equation}
is monotonically decreasing in $\lambda$ and $\zeta_p$ is a root of $f$. This implies
\begin{equation}\begin{aligned}
    &f(x) \leq 0 \iff x \geq \zeta_p\\
    \implies & f(b_i) \leq 0 \iff b_i \geq \zeta_p\\
    \implies &\{i|b_i\geq \zeta_p\} =\{i|f(b_i)\leq 0\}\\
    \implies &\chi_p = \max\{i|b_i\geq \zeta_p\} = \max\{i|f(b_i)\leq 0\}.
\end{aligned}\end{equation}
Hence, things follows by putting back in the definition of $f$.
\item We have, 
\[\alpha^{p} = \sum_{i=1}^A(b_i - \zeta_p)^{p}\mathbf{1}(b_i\geq\zeta_p),\quad \text{and}\quad \chi_p = \max\{i|b_i\geq \zeta_p\}.\]
Combining both we have
\[\alpha^{p} = \sum_{i=1}^{\chi_p}(b_i - \zeta_p)^{p}.\]
And the other part follows directly.

\item Continuity and montonocity of the function $\sum_{b_i\geq x}(b_i - x)^{p} $ is trivial. Now observe that
$\sum_{b_i\geq b_1}(b_i - b_1)^{p} = 0$ and $ \sum_{b_i\geq b_1-\alpha}(b_i - (b_1-\alpha))^{p} \geq \alpha^p$,
so it implies that it is equal to $\alpha^p$ in the range $[b_1-\alpha,b_1]$.
\item Recall that the $\zeta_p$ is the solution to the following equation
\[\alpha^{p} = \sum_{b_i\geq x}(b_i - x)^{p}.\]
 And from the definition of $\chi_p$, we have \[\alpha^{p} < \sum_{i=1}^{ \chi_p +1}(b_i - b_{\chi_p +1})^{p} = \sum_{b_i\geq b_{\chi_p+1} }(b_i - b_{\chi_p+1 })^{p},\quad \text{and}\]
 \[ \alpha^{p} \geq \sum_{i=1}^{\chi_p} (b_i - b_{\chi_p })^{p} =\sum_{b_i\geq b_{\chi_p}} (b_i - b_{\chi_p })^{p}.\]
 
So from continuity, we infer the root $\zeta_p$ must lie between $[b_{\chi_p+1},b_{\chi}]$.
\item We  prove the first direction, and assume we have 
\begin{equation}\begin{aligned}
    & k \leq \chi_p \\
    \implies & \sum_{i=1}^k(b_i -b_k)^p \leq \alpha^p \qquad \text{(from definition of $\chi_p$)}.\\
\end{aligned}
\end{equation}
Observe the function $f(x):=\sum_{i=1}^k(b_i -x)^p$ is monotically decreasing in the range $(-\infty,b_k]$. Further, $f(b_k) \leq \alpha^p$  and $\lim_{x\to-\infty}f(x) = \infty$, so from the continuity argument there must exist a value $y\in (-\infty,b_k]$ such that $f(y) = \alpha^p$. This implies that 
\[\sum_{i=1}^k(b_i -y)^p \leq \alpha^p,\quad\text{and}\quad y\leq b_k.\]
Hence, explicitly showed the existence of the solution. Now, we move on to the second direction, and assume there exist $x$ such that
\[\sum_{i=1}^k(b_i -x)^p = \alpha^p,\quad\text{and}\quad x\leq b_k.\]
\[\implies \sum_{i=1}^k(b_i -b_k)^p \leq \alpha^p, \qquad \text{(as $x\leq b_k\leq b_{k-1}\cdots \leq b_1$)}\]
\[\implies k \leq \chi_p.\]

 \item We have $k \leq \chi_p$ and $\lambda_k$ such that 
 \begin{equation}\begin{split}
     \alpha^{p} 
     &= \sum_{i=1}^k(b_i - \lambda_k)^{p},\quad\text{and}\quad \lambda_k \leq b_k, \qquad \text{(from above item)}\\
      &  \geq \sum_{i=1}^k(b_i - b_{k+1})^{p},\qquad \text{(as $\lambda_k \leq b_{k+1} \leq b_k$)}\\
      & \geq \sum_{i=1}^{k+1}(b_i - b_{k+1})^{p},\qquad \text{(addition of $0$).}
 \end{split}
 \end{equation}
 From the definition of $\chi_p$, we get $k+1\leq \chi_p$. 
 
 \item We are given 
 \[ \sum_{i=1}^k(b_i-\lambda_k)^p = \alpha^p\]
 \[\implies \sum_{i=1}^k(b_i-b_{k+1})^p > \alpha^p,\qquad \text{(as $\lambda_k > b_{k+1}$)}\]
  \[\implies \sum_{i=1}^{k+1}(b_i-b_{k+1})^p > \alpha^p,\qquad \text{(addition of zero)}\]
    \[ \implies k+1 > \chi_p.\]

\end{enumerate}
\end{proof}

\subsubsection{Special case: $p=1$}\label{app:L1waterpouringSC}
For $p=1$, by definition, we have
\begin{equation}\begin{aligned}
  \zeta_1& = \max_{c}-\alpha \lVert c\rVert_{\infty} + \langle c,b\rangle \qquad \text{such that }\qquad \sum_{a\in \mathcal{A}}c_a = 1,\quad c\succeq 0.\\
\end{aligned}\end{equation}
And $\chi_1$ is the optimal number of actions, that is 
\[\alpha = \sum_{i=1}^{\chi_1}(b_i - \zeta_1) \]
\[\implies \zeta_1 =\frac{\sum_{i=1}^{\chi_1} b_i - \alpha}{\chi_1}. \]
Let $\lambda_k$ be the such that  
\[\alpha = \sum_{i=1}^{k}(b_i - \lambda_k) \]
\[\implies \lambda_k =\frac{\sum_{i=1}^{k} b_i - \alpha}{k}. \]

\begin{proposition}
 \[\zeta_1 = \max_{k}\lambda_k\]
\end{proposition}
 \begin{proof}
 From lemma \ref{LpwaterPouring}, we have 
 \[\lambda_1 \leq \lambda_2 \cdots \leq\lambda_{\chi_1}.\]
Now, we have  
\begin{equation}\begin{split}\label{eq:Telelambda}
    \lambda_{k}-\lambda_{k+m} &= \frac{\sum_{i=1}^kb_i- \alpha}{k}-\frac{\sum_{i=1}^{k+m}b_i- \alpha}{k+m}\\ 
    &= \frac{\sum_{i=1}^kb_i- \alpha}{k}-\frac{\sum_{i=1}^{k}b_i- \alpha}{k+m} -\frac{\sum_{i=1}^m b_{k+i}}{k+m}\\ 
    &= \frac{m(\sum_{i=1}^kb_i- \alpha}{k(k+m))} -\frac{\sum_{i=1}^m b_{k+i}}{k+m}\\
    &= \frac{m}{k+m}(\frac{\sum_{i=1}^kb_i- \alpha}{k} -\frac{\sum_{i=1}^m b_{k+i}}{m})\\
    &= \frac{m}{k+m}(\lambda_k -\frac{\sum_{i=1}^m b_{k+i}}{m})\\
\end{split}
\end{equation}
 From lemma \ref{LpwaterPouring}, we also know the stopping criteria for $\chi_1$, that is 
 \[\lambda_{\chi_1} > b_{\chi_1 +1}\]
 \[\implies \lambda_{\chi_1} > b_{\chi_1 + i}, \qquad  i\geq 1,\qquad \text{(as $b_i$ are in descending order)}  \]
 \[\implies \lambda_{\chi_1} > \frac{\sum_{i=1}^m b_{\chi_1+i}}{m}, \qquad \forall m\geq 1.\]
 Combining it with the \eqref{eq:Telelambda}, for all $m
 \geq 0$ , we get 
 \begin{equation}\begin{split}
 \lambda_{\chi_1}-\lambda_{\chi_1+m} &= \frac{m}{\chi_1+m}(\lambda_{\chi_1} -\frac{\sum_{i=1}^m b_{\chi_1+i}}{m})\\
 &\geq 0\\
 \implies \lambda_{\chi_1}&\geq \lambda_{\chi_1+m} 
 \end{split}
 \end{equation}
 Hence, we get the desired result, 
\[\zeta_1 = \lambda_{\chi_1} = \max_{k}\lambda_k.\]
 \end{proof} 

\subsubsection{Special case: $p=\infty$}
For $p=\infty$, by definition, we have
\begin{equation}\begin{aligned}
  \zeta_\infty(b)&= \max_{c}-\alpha \lVert c\rVert_{1} + \langle c,b\rangle \qquad \text{such that }\qquad \sum_{a\in \mathcal{A}}c_a = 1,\quad c\succeq 0.\\
= &\max_{c}-\alpha + \langle c,b\rangle \qquad \text{such that }\qquad \sum_{a\in \mathcal{A}}c_a = 1,\quad c\succeq 0. \\
= & -\alpha + \max_{i}b_i
\end{aligned}\end{equation}

\subsubsection{Special case: $p=2$}
The problem is discussed in great details in \cite{anava2016k}, here we outline the proof.
For $p=2$, we have
\begin{equation}\begin{aligned}
  \zeta_2&= \max_{c}-\alpha \lVert c\rVert_{2} + \langle c,b\rangle \qquad \text{such that }\qquad \sum_{a\in \mathcal{A}}c_a = 1,\quad c\succeq 0.\\
\end{aligned}\end{equation}
Let $\lambda_k$ be the solution of the following equation
\begin{equation}\begin{split}
\alpha^{2} &= \sum_{i=1}^k(b_i - \lambda)^{2},\qquad \lambda \leq b_k\\
    &=  k\lambda^2 -2\sum_{i=1}^k\lambda b_i. +\sum_{i=1}^k(b_i)^2, \qquad \lambda \leq b_k \\
 \implies \lambda_k &= \frac{\sum_{i=1}^k b_i \pm \sqrt{(\sum_{i=1}^kb_i)^2 - k(\sum_{i=1}^k(b_i)^2 - \alpha^{2} )}}{k}, \quad \text{and } \quad \qquad \lambda_k \leq b_k \\
  &= \frac{\sum_{i=1}^k b_i - \sqrt{(\sum_{i=1}^kb_i)^2 - k(\sum_{i=1}^k(b_i)^2 - \alpha^{2} )}}{k}\\
  &= \frac{\sum_{i=1}^k b_i}{k} - \sqrt{\alpha^{2} -\sum_{i=1}^k(b_i -\frac{\sum_{i=1}^kb_i}{k})^2}\\
\end{split}
\end{equation}
From lemma \ref{LpwaterPouring}, we know
 \[\lambda_1 \leq \lambda_2 \cdots \leq \lambda_{\chi_2} = \zeta_2\]
 where $\chi_2$ calculated in two ways: a) 
 \[\chi_2 = \max_{m}\{m| \sum_{i=1}^m(b_i-b_m)^2\leq \alpha^2\}\]
 b) \[\chi_2 = \min_{m}\{m |\lambda_{m}\leq b_{m+1}\}\]
 We proceed greedily until stopping condition is met in lemma \ref{LpwaterPouring}. Concretely, it is illustrated in algorithm \ref{alg:f2}.

\subsection{$L_1$ Water Pouring lemma}\label{app:L1waterpouringLemma}
In this section, we re-derive the above water pouring lemma for 
$p=1$ from scratch, just for sanity check. As in the above proof, there is a possibility of some breakdown, as we had take limits $q\to\infty$. We will see that all the above results for $p=1$ too.\\

Let $b \in\mathbb{R}^A$ be such that its components are in decreasing order, i,e $b_{i}\geq b_{i+1}$ and
\begin{equation}\label{eq:LpwaterPouringlemma}
 \zeta_1 := \max_{c}-\alpha \lVert c\rVert_\infty + \langle c,b\rangle \qquad \text{such that }\qquad \sum_{i=1}^Ac_i = 1,\quad c_i\geq 0,\quad \forall i.   
\end{equation}
 Lets fix any vector $c\in\mathbb{R}^A$, and let $k_1 :=\lfloor\frac{1}{\max_{i}c_i}\rfloor$ and let \[c^1_i = \begin{cases}\max_{i}c_i\qquad&\text{if}\quad i\leq k_1\\
1- k_1\max_{i}c_i\qquad&\text{if}\quad i= k_1+1\\
0\qquad&\text{else}\\\end{cases}  \]
Then we have,
 \begin{equation}\begin{aligned}
     -\alpha \lVert c\rVert_\infty + \langle c,b\rangle =& -\alpha \max_{i}c_i + \sum_{i=1}^A c_ib_i\\
     \leq&-\alpha \max_{i}c_i + \sum_{i=1}^A c^1_ib_i, \qquad\text{(recall $b_i$ is in decreasing order)}\\
     =& -\alpha \lVert c^1\rVert_\infty + \langle c^1,b\rangle
 \end{aligned}\end{equation}
  Now, lets define $c^2\in\mathbb{R}^A$. Let 
  \[k_2 = \begin{cases}k_1+1\qquad&\text{if}\quad \frac{\sum_{i=1}^{k_1}b_i -\alpha}{k_1} \leq b_{k+1}\\
k_1\qquad &\text{else}\\\end{cases}  \]
and let $c^2_i = \frac{\mathbf{1}(i\leq k_2)}{k_2}$. Then we have,
\begin{equation}\begin{aligned}
     -\alpha \lVert c^1\rVert_\infty + \langle c^1,b\rangle =&-\alpha \max_{i}c_i + \sum_{i=1}^A c^1_ib_i\\
     =&-\alpha \max_{i}c_i + \sum_{i=1}^{k_1}\max_{i}c_i b_i +(1-k_1\max_{i}c_i) b_{k_1 +1},\qquad \text{(by definition of $c^1$)}\\
    =&(\frac{-\alpha  +  \sum_{i=1}^{k_1}b_i}{k_1})k_1\max_{i}c_i + b_{k_1 +1}(1-k_1\max_{i}c_i),\qquad \text{(re-arranging)}\\
    \leq & \frac{-\alpha  +  \sum_{i=1}^{k_2}b_i}{k_2}\\
     =& -\alpha \lVert c^2\rVert_\infty + \langle c^2,b\rangle
 \end{aligned}\end{equation}
 The last inequality comes from the definition of $k_2$ and $c^2$. So we conclude that a optimal solution is uniform over some actions, that is
 \begin{equation}\begin{split}
     \zeta_1 =& \max_{c\in\mathcal{C}}-\alpha \lVert c\rVert_\infty + \langle c,b\rangle\\
     =& \max_{k}\bigm(\frac{-\alpha+\sum_{i=1}^kb_i}{k}\bigm)\\
 \end{split}
 \end{equation}
 where $\mathcal{C}:=\{c^k\in\mathbb{R}^A|c^k_i = \frac{\mathbf{1}(i\leq k)}{k} \}$ is set of uniform actions. Rest all the properties follows same as $L_p$ water pouring lemma.

\section{Robust Value Iteration (Main)} \label{app:srLp}
In this section, we will discuss the main results from the paper except for time complexity results. It contains the proofs of the results presented in the main body and also some other corollaries/special cases. \\

\subsection{$(\mathtt{sa})$-rectangular robust policy evaluation and improvement}
 \begin{theorem*} $(\mathtt{sa})$-rectangular $L_p$ robust Bellman operator is equivalent to reward regularized (non-robust) Bellman operator, that is 
\begin{equation*}\begin{aligned}
    (\mathcal{T}^\pi_{\mathcal{U}^{\mathtt{sa}}_p} v)(s)  =& \sum_{a}\pi(a|s)[  -\alpha_{s,a} -\gamma\beta_{s,a}\kappa_q(v)  +R_0(s,a) +\gamma \sum_{s'}P_0(s'|s,a)v(s')], \qquad \text{and}\\
    (\mathcal{T}^*_{\mathcal{U}^{\mathtt{sa}}_p} v)(s)  =& \max_{a\in\mathcal{A}}[  -\alpha_{s,a} -\gamma\beta_{s,a}\kappa_q(v)  +R_0(s,a) +\gamma \sum_{s'}P_0(s'|s,a)v(s')],
\end{aligned}\end{equation*}
where $\kappa_p$ is defined in \eqref{def:kp}.
\end{theorem*}
\begin{proof}
From definition robust Bellman operator and  $\mathcal{U}^{\mathtt{sa}}_p = (R_0 +\mathcal{R})\times(P_0 +\mathcal{P})$, we have,
\begin{equation}\begin{aligned}
   (&\mathcal{T}^\pi_{\mathcal{U}^{\mathtt{sa}}_p} v)(s)=\min_{{R,P\in\mathcal{U}^{\mathtt{sa}}_p}}\sum_{a}\pi(a|s)\Bigm[R(s,a) + \gamma \sum_{s'}P(s'|s,a)v(s')\Bigm] \\
    &=\sum_{a}\pi(a|s)\Bigm[R_0(s,a) + \gamma \sum_{s'}P_0(s'|s,a)v(s')\Bigm]+   \\
    &\qquad\qquad \min_{{p\in\mathcal{P}},r\in\mathcal{R}}\sum_{a}\pi(a|s)\Bigm[r(s,a) +\gamma \sum_{s'}p(s'|s,a)v(s')\Bigm] ,\\
    & \quad \text{(from $(\mathtt{sa})$-rectangularity, we get)}\\
    &=\sum_{a}\pi(a|s)\Bigm[R_0(s,a) + \gamma \sum_{s'}P_0(s'|s,a)v(s')\Bigm]+ \\
    &\qquad\qquad \sum_{a}\pi(a|s)\underbrace{\min_{{p_{s,a}\in\mathcal{P}_{sa}},r_{s,a}\in\mathcal{R}_{s,a}}\Bigm[r_{s,a} + \gamma \sum_{s'}p_{s,a}(s')v(s')\Bigm]}_{:=\Omega_{sa}(v)} 
\end{aligned}\end{equation}
Now we focus on regularizer function $\Omega$, as follows
\begin{equation}\begin{aligned}
    \Omega_{sa}(v)=&\min_{{p_{s,a}\in\mathcal{P}_{s,a}},r_{s,a}\in\mathcal{R}_{s,a}}\Bigm[r_{s,a} + \gamma \sum_{s'}p_{s,a}(s')v(s')\Bigm] \\
    =&\min_{r_{s,a}\in\mathcal{R}_{s,a}}r_{s,a} + \gamma\min_{{p_{s,a}\in\mathcal{P}_{sa}}} \sum_{s'}p_{s,a}(s')v(s') \\
    &=- \alpha_{s,a} +\gamma\min_{\lVert p_{sa}\rVert_p\leq \beta_{s,a},\sum_{s'}p_{sa}(s')=0} \langle p_{s,a}, v\rangle,\\
    =&- \alpha_{s,a} -\gamma \beta_{s,a}\kappa_q(v), \qquad \text{(from lemma \ref{regfn}).}\\
\end{aligned}\end{equation}
Putting back, we have
\begin{equation*}
    (\mathcal{T}^\pi_{\mathcal{U}^{\mathtt{sa}}_p} v)(s)=\sum_{a}\pi(a|s)\Bigm[- \alpha_{s,a} -\gamma \beta_{s,a}\kappa_q(v)+R_0(s,a) + \gamma \sum_{s'}P_0(s'|s,a)v(s')\Bigm]
\end{equation*}
Again, reusing above results in optimal robust operator, we have
\begin{equation}\begin{aligned}
    (\mathcal{T}^*_{\mathcal{U}^{\mathtt{sa}}_p} v)(s) &= \max_{\pi_s\in\Delta_\mathcal{A}}\min_{{R,P\in\mathcal{U}^{\mathtt{sa}}_p}}\sum_{a}\pi_s(a)\Bigm[R(s,a) + \gamma \sum_{s'}P(s'|s,a)v(s')\Bigm]\\
    &=\max_{\pi_s\in\Delta_\mathcal{A}}\sum_{a}\pi_s(a)\Bigm[- \alpha_{s,a} -\gamma \beta_{s,a}\kappa_p(v)+R_0(s,a) + \gamma \sum_{s'}P_0(s'|s,a)v(s')\Bigm]\\
    &=\max_{a\in\mathcal{A}}\Bigm[- \alpha_{s,a} -\gamma \beta_{s,a}\kappa_q(v)+R_0(s,a) + \gamma \sum_{s'}P_0(s'|s,a)v(s')\Bigm]
\end{aligned}\end{equation}
The claim is proved.
\end{proof}

\subsection{$\mathtt{S}$-rectangular robust policy evaluation}\label{app:sLprpe}
\begin{theorem*} $\mathtt{S}$-rectangular $L_p$ robust Bellman operator is equivalent to reward regularized (non-robust) Bellman operator, that is 
\begin{equation*}
    (\mathcal{T}^\pi_{\mathcal{U}^s_p} v)(s)  =   -\Bigm(\alpha_s +\gamma\beta_{s}\kappa_q(v)\Bigm)\lVert\pi(\cdot|s)\rVert_q  +\sum_{a}\pi(a|s)\Bigm(R_0(s,a) +\gamma \sum_{s'}P_0(s'|s,a)v(s')\Bigm)
\end{equation*}
where $\kappa_p$ is defined in \eqref{def:kp} and $\lVert \pi(\cdot|s)\rVert_q$ is $q$-norm of the vector $\pi(\cdot|s)\in\Delta_{\mathcal{A}}$.
\end{theorem*}

\begin{proof}
From definition of robust Bellman operator and  $\mathcal{U}^{\mathtt{s}}_p = (R_0 +\mathcal{R})\times(P_0 +\mathcal{P})$, we have
\begin{equation}\begin{aligned}
    (&\mathcal{T}^\pi_{\mathcal{U}^s_p} v)(s)  = 
    \min_{{R,P\in\mathcal{U}^s_p}}\sum_{a}\pi(a|s)\Bigm[R(s,a) + \gamma \sum_{s'}P(s'|s,a)v(s')\Bigm] \\
    &=\sum_{a}\pi(a|s)\Bigm[\underbrace{R_0(s,a) + \gamma \sum_{s'}P_0(s'|s,a)v(s')}_{\text{nominal values}}\Bigm] \\
    &\qquad\qquad +\min_{{p\in\mathcal{P}},r\in\mathcal{R}}\sum_{a}\pi(a|s)\Bigm[r(s,a) + \gamma \sum_{s'}p(s'|s,a)v(s')\Bigm]\\
    &\qquad \text{(from $\mathtt{s}$-rectangularity we have)}\\
    &=\sum_{a}\pi(a|s)\Bigm[R_0(s,a) + \gamma \sum_{s'}P_0(s'|s,a)v(s')\Bigm] \\
    & \qquad\qquad +  \underbrace{\min_{{p_s\in\mathcal{P}_s},r_s\in\mathcal{R}_s}\sum_{a}\pi(a|s)\Bigm[r_s(a) + \gamma \sum_{s'}p_s(s'|a)v(s')\Bigm]}_{:=\Omega_s(\pi_s,v)}
\end{aligned}\end{equation}
where we denote $\pi_s(a) = \pi(a|s)$ as a shorthand. Now we calculate the regularizer function as follows

\begin{equation}\begin{aligned}
    \Omega_s(\pi_s,v):=&\min_{r_s\in\mathcal{R}_s,p_s\in\mathcal{P}_s}\langle r_s + \gamma v^T p_s,\pi_s \rangle =\min_{r_s\in\mathcal{R}_s}\langle r_s,\pi_s\rangle + \gamma\min_{p_s\in\mathcal{P}_s} v^T p_s\pi_s   \\
    &=-\alpha_s\lVert \pi_s\rVert_q +\gamma\min_{p_s\in\mathcal{P}_s} v^T p_s\pi_s, \qquad \text{(using Holders inequality, where $\frac{1}{p} + \frac{1}{q} = 1$  )}\\
    =&- \alpha_s\lVert \pi_s\rVert_q +\gamma \min_{p_s\in\mathcal{P}_s}\sum_{a}\pi_s(a)\langle p_{s,a}, v\rangle\\
    =&- \alpha_s\lVert \pi_s\rVert_q +\gamma \min_{\sum_{a}(\beta_{s,a})^p \leq (\beta_s)^p}\quad \min_{\lVert p_{sa}\rVert_p\leq \beta_{s,a},\sum_{s'}p_{sa}(s')=0 }\quad\sum_{a}\pi_s(a)\langle p_{s,a}, v\rangle  \\
    =&- \alpha_s\lVert \pi_s\rVert_q +\gamma \min_{\sum_{a}(\beta_{s,a})^p \leq (\beta_s)^p}\sum_{a}\pi_s(a)\quad \min_{\lVert p_{sa}\rVert_p\leq \beta_{s,a},\sum_{s'}p_{sa}(s')=0 }\quad\langle p_{s,a}, v\rangle  \qquad \text{}\\
    =&- \alpha_s\lVert \pi_s\rVert_q +\gamma \min_{\sum_{a}(\beta_{sa})^p \leq (\beta_s)^p}\sum_{a}\pi_s(a)(-\beta_{sa}\kappa_p(v))  \qquad \text{ ( from lemma \ref{regfn})}\\
     =&- \alpha_s\lVert \pi_s\rVert_q -\gamma \kappa_q(v)\max_{\sum_{a}(\beta_{sa})^p \leq (\beta_s)^p}\sum_{a}\pi_s(a)\beta_{sa}  \qquad \text{}\\
      =&- \alpha_s\lVert \pi_s\rVert_q -\gamma \kappa_p(v)\lVert \pi_s\rVert_q\beta_{s}  \qquad \text{(using Holders)}\\
      =&- (\alpha_s +\gamma\beta_{s}\kappa_q(v))\lVert \pi_s\rVert_q .  \\
\end{aligned}\end{equation}
Now putting above values in robust operator, we have
\begin{equation*}\begin{aligned}
    (\mathcal{T}^\pi_{\mathcal{U}^s_p} v)(s) 
     &=- \Bigm(\alpha_s +\gamma\beta_{s}\kappa_q(v)\Bigm)\lVert \pi(\cdot|s)\rVert_q + \sum_{a}\pi(a|s)\Bigm(R_0(s,a) + \gamma \sum_{s'}P_0(s'|s,a)v(s')\Bigm).\\
\end{aligned}\end{equation*}
\end{proof}

\subsection{$\mathtt{s}$-rectangular robust policy improvement}

Reusing robust policy evaluation results in section \ref{app:sLprpe}, we have
\begin{equation}\begin{aligned}
    (\mathcal{T}^*_{\mathcal{U}^{\mathtt{s}}_p} v)(s) &= \max_{\pi_s\in\Delta_\mathcal{A}}\min_{{R,P\in\mathcal{U}^{\mathtt{sa}}_p}}\sum_{a}\pi_s(a)\Bigm[R(s,a) + \gamma \sum_{s'}P(s'|s,a)v(s')\Bigm]\\
    &=\max_{\pi_s\in\Delta_{\mathcal{A}}}\Bigm[  -(\alpha_s +\gamma\beta_{s}\kappa_q(v))\lVert \pi_s\rVert_q  +\sum_{a}\pi_s(a)(R(s,a)  + \gamma \sum_{s'}P(s'|s,a) v(s'))\Bigm].
\end{aligned}\end{equation}
Observe that, we have the following form
\begin{equation}\label{app:eq:wp:ref}
(\mathcal{T}^*_{\mathcal{U}^{\mathtt{s}}_p} v)(s)= \max_{c}-\alpha \lVert c\rVert_q + \langle c,b\rangle \qquad \text{such that }\qquad \sum_{i=1}^A c_i = 1,\quad c\succeq 0,
\end{equation}
where $\alpha =\alpha_s +\gamma\beta_{s}\kappa_q(v) $ and $b_i =R(s,a_i)  + \gamma \sum_{s'}P(s'|s,a_i) v(s') $. Now all the results below, follows from water pouring lemma ( lemma \ref{LpwaterPouring}).

\begin{theorem*}(Policy improvement) The optimal robust Bellman operator can be evaluated in following ways.
\begin{enumerate}
    \item $(\mathcal{T}^*_{\mathcal{U}^s_p}v)(s)$ is the solution of the following equation that can be found using binary search between $\bigm[\max_{a}Q(s,a)-\sigma, \max_{a}Q(s,a)\bigm]$,

\begin{equation}
    \sum_{a}\bigm(Q(s,a) - x\bigm)^p\mathbf{1}\bigm( Q(s,a) \geq x\bigm)  = \sigma^p.
\end{equation}
\item $(\mathcal{T}^*_{\mathcal{U}^s_p}v)(s)$ and $\chi_p(v,s)$ can also be computed through algorithm \ref{alg:fp}.
\end{enumerate}
where $\sigma = \alpha_s + \gamma\beta_s\kappa_q(v),$ and $Q(s,a)= R_0(s,a) + \gamma\sum_{s'} P_0(s'|s,a)v(s')$. \end{theorem*}
\begin{proof}
The first part follows from lemma \ref{LpwaterPouring}, point \ref{app:wp:zeta}. The second part follows from lemma \ref{LpwaterPouring}, point \ref{app:wp:greedyInclusion} (greedy inclusion ) and point \ref{app:wp:stoppingCondition} (stopping condition).
\end{proof}

\begin{theorem*}(Go To Policy) The greedy policy $\pi$ w.r.t. value function $v$, defined as $\mathcal{T}^*_{\mathcal{U}^s_p}v =\mathcal{T}^\pi_{\mathcal{U}^s_p}v$ is a threshold policy. It takes only those actions that has positive advantage, with probability proportional to $(p-1)$th power of its advantage. That is
\[\pi(a|s)\propto (A(s,a))^{p-1}\mathbf{1}(A(s,a)\geq 0),\]

where $A(s,a)= R_0(s,a) + \gamma\sum_{s'} P_0(s'|s,a)v(s') -  (\mathcal{T}^*_{\mathcal{U}^s_p}v)(s)$.
\end{theorem*}
\begin{proof}
Follows from lemma \ref{LpwaterPouring}, point \ref{app:wpl:policy}.
\end{proof}

\begin{property*} $\chi_p(v,s)$ is number of actions that has positive advantage, that is 
\[\chi_p(v,s) = \Bigm\lvert\bigm\{a \mid (\mathcal{T}^*_{\mathcal{U}^s_p}v)(s) \leq  R_0(s,a) + \gamma\sum_{s'} P_0(s'|s,a)v(s')\bigm\}\Bigm\rvert.\]
\end{property*}
\begin{proof}

Follows from lemma \ref{LpwaterPouring}, point \ref{app:wp:chi}.
\end{proof}

\begin{property*}( Value vs Q-value) $(\mathcal{T}^*_{\mathcal{U}^s_p}v)(s)$ is bounded by the Q-value of $\chi$th and $(\chi+1)$th actions. That is  
\[Q(s, a_{\chi+1}) < (\mathcal{T}^*_{\mathcal{U}^s_p}v)(s) \leq Q(s, a_\chi),\qquad\text{ where}\quad \chi = \chi_p(v,s),\]
 $Q(s,a) =R_0(s,a) +\gamma \sum_{s'}P_0(s'|s,a)v(s')$, and $Q(s,a_1)\geq Q(s,a_2),\cdots Q(s,a_A)$.
\end{property*}
\begin{proof}
Follows from lemma \ref{LpwaterPouring}, point \ref{app:wp:zetaBound}.
\end{proof}

\begin{corollary} For $p=1$, the optimal policy $\pi_1$ w.r.t. value function $v$ and uncertainty set $\mathcal{U}^s_1$, can be computed directly using $\chi_1(s)$ without calculating advantage function. That is 
\[\pi_1(a^s_i|s) = \frac{\mathbf{1}(i\leq \chi_1(s))}{\chi_1(s)}.\]
\end{corollary}
\begin{proof}
Follows from theorem \ref{rs:GoToPolicy} by putting $p=1$. Note that it can be directly obtained using $L_1$ water pouring lemma (see section \ref{app:L1waterpouringLemma})
\end{proof}

\begin{corollary}\label{rs:policy:inf} (For $p=\infty$) The optimal policy $\pi$ w.r.t. value function $v$ and uncertainty set $\mathcal{U}^s_\infty$   (precisely $\mathcal{T}^*_{\mathcal{U}^s_\infty}v = \mathcal{T}^\pi_{\mathcal{U}^s_\infty}v$), is to play the best response, that is 
\[\pi(a|s) = \frac{\mathbf{1}(a\in arg\max_{a}Q(s,a))}{\bigm \lvert arg\max_{a}Q(s,a)\bigm\rvert}.\]
In case of tie in the best response, it is optimal to play any of the best responses with any probability. 
\end{corollary}
\begin{proof}
Follows from theorem \ref{rs:GoToPolicy} by taking limit $p\to\infty$. \end{proof}

\begin{corollary} For $p=\infty$, $\mathcal{T}^*_{\mathcal{U}^s_p}v$, the robust optimal Bellman operator evaluation can be obtained in closed form. That is 
\[(\mathcal{T}^*_{\mathcal{U}^s_\infty}v)(s) = \max_{a}Q(s,a) - \sigma,\]
where $\sigma = \alpha_s + \gamma\beta_s\kappa_1(v), Q(s,a) = R_0(s,a) + \gamma\sum_{s'} P_0(s'|s,a)v(s')$.
\end{corollary}
\begin{proof}
Let $\pi$ be such that
\[\mathcal{T}^*_{\mathcal{U}^s_\infty}v = \mathcal{T}^\pi_{\mathcal{U}^s_\infty}v.\]
This implies
\begin{equation}\begin{aligned}
    (\mathcal{T}^*_{\mathcal{U}^{\mathtt{s}}_p} v)(s) &= \min_{{R,P\in\mathcal{U}^{\mathtt{sa}}_p}}\sum_{a}\pi(a|s)\Bigm[R(s,a) + \gamma \sum_{s'}P(s'|s,a)v(s')\Bigm]\\
    &=  -(\alpha_s +\gamma\beta_{s}\kappa_p(v))\lVert \pi(\cdot|s)\rVert_q  +\sum_{a}\pi(a|s)(R(s,a)  + \gamma \sum_{s'}P(s'|s,a) v(s')).
\end{aligned}\end{equation}
From corollary \ref{rs:policy:inf}, we know the that $\pi$ is deterministic best response policy. Putting this we get the desired result.\\
There is a another way of proving this, using theorem \ref{rs:rve} by taking limit $p\to \infty$ carefully as
\begin{equation}
    \lim_{p\to\infty}\sum_{a}\left(Q(s,a) - \mathcal{T}^*_{\mathcal{U}^{\mathtt{s}}_p} v)(s)\right)^p\mathbf{1}\left( Q(s,a) \geq \mathcal{T}^*_{\mathcal{U}^{\mathtt{s}}_p} v)(s)\right))^\frac{1}{p} = \sigma,
\end{equation}
where $\sigma = \alpha_s + \gamma\beta_s\kappa_1(v)$.
\end{proof}

\begin{corollary} For $p=1$, the robust optimal Bellman operator $\mathcal{T}^*_{\mathcal{U}^s_p}$, can be computed in closed form. That is
\[(\mathcal{T}^*_{\mathcal{U}^s_p}v)(s) = \max_{k}\frac{\sum_{i=1}^{k} Q(s,a_i) - \sigma}{k},\]
where $\sigma = \alpha_s + \gamma\beta_s\kappa_\infty(v), Q(s,a) = R_0(s,a) + \gamma\sum_{s'} P_0(s'|s,a)v(s')$, and $Q(s,a_1)\geq Q(s,a_2),\geq \cdots \geq Q(s,a_A).$\end{corollary}
\begin{proof}
Follows from section \ref{app:L1waterpouringSC}. 
\end{proof}

\begin{corollary} \label{rs:f1f2}
The $\mathtt{s}$ rectangular $L_p$ robust Bellman operator can be evaluated for $p =1 ,2$ by algorithm \ref{alg:f1} and algorithm \ref{alg:f2} respectively.
\end{corollary}
\begin{proof}
It follows from the algorithm \ref{alg:fp}, where we solve the linear equation and quadratic equation for $p=1,2$ respectively. For $p=2$, it can be found in \cite{anava2016k}.
\end{proof}

\begin{algorithm}
\caption{Algorithm to compute $S$-rectangular $L_2$ robust optimal Bellman Operator}\label{alg:f2}
\begin{algorithmic} [1]
 \STATE \textbf{Input:} $\sigma = \alpha_s +\gamma\beta_s\kappa_2(v), \qquad Q(s,a) = R_0(s,a) + \gamma\sum_{s'} P_0(s'|s,a)v(s')$.
 \STATE \textbf{Output} $(\mathcal{T}^*_{\mathcal{U}^s_2}v)(s), \chi_2(v,s)$
\STATE Sort $Q(s,\cdot)$ and label actions such that $Q(s,a_1)\geq Q(s,a_2), \cdots$.
\STATE Set initial value guess $\lambda_1 = Q(s,a_1)-\sigma$ and counter $k=1$.
\WHILE{$k \leq A-1  $ and $\lambda_k \leq Q(s,a_k)$}
    \STATE Increment counter: $k = k+1$
    \STATE Update value estimate: \[\lambda_k = \frac{1}{k}\Bigm[\sum_{i=1}^{k}Q(s,a_i) - \sqrt{k\sigma^2 + (\sum_{i=1}^{k}Q(s,a_i))^2 - k\sum_{i=1}^{k}(Q(s,a_i))^2}\Bigm]\]
\ENDWHILE
\STATE Return: $\lambda_k, k $
\end{algorithmic}
\end{algorithm}

\begin{algorithm}
\caption{Algorithm to compute $S$-rectangular $L_1$ robust optimal Bellman Operator}\label{alg:f1}
\begin{algorithmic} [1]
 \STATE \textbf{Input:} $\sigma = \alpha_s +\gamma\beta_s\kappa_\infty(v), \qquad Q(s,a) = R_0(s,a) + \gamma\sum_{s'} P_0(s'|s,a)v(s')$.
 \STATE \textbf{Output} $(\mathcal{T}^*_{\mathcal{U}^s_1}v)(s), \chi_1(v,s)$
\STATE Sort $Q(s,\cdot)$ and label actions such that $Q(s,a_1)\geq Q(s,a_2), \cdots$.
\STATE Set initial value guess $\lambda_1 = Q(s,a_1)-\sigma$ and counter $k=1$.
\WHILE{$k \leq A-1  $ and $\lambda_k \leq Q(s,a_k)$}
    \STATE Increment counter: $k = k+1$
    \STATE Update value estimate: \[\lambda_k = \frac{1}{k}\Bigm[\sum_{i=1}^kQ(s,a_i) -\sigma \Bigm]\]
\ENDWHILE
\STATE Return: $\lambda_k, k $
\end{algorithmic}
\end{algorithm}

\section{Time Complexity} \label{app:timeComplexitySection}

In this section, we will discuss time complexity of various robust MDPs and compare it with time complexity of non-robust MDPs. We assume that we have the  knowledge of nominal transition kernel and nominal reward function for robust MDPs, and in case of non-robust MDPs, we assume the knowledge of the transition kernel and reward function. We divide the discussion into various parts depending upon their similarity. 

\subsection{Exact Value Iteration: Best Response}
In this section, we will discuss non-robust MDPs, $(\mathtt{sa})$-rectangular $L_1/L_2/L_\infty$ robust MDPs and $\mathtt{s}$-rectangular $L_\infty$ robust MDPs. They all have  a common theme for value iteration as follows, for the value function $v$, their Bellman operator ( $\mathcal{T}$) evaluation is done as
\begin{equation}\begin{aligned}
   (\mathcal{T} v)(s) =& \underbrace{\max_{a}}_{\text{action cost}}\Bigm[R(s,a)+ \alpha_{s,a}\underbrace{\kappa(v)}_{\text{reward penalty/cost}} + \gamma \underbrace{\sum_{s'}P(s'|s,a)v(s')}_{\text{sweep}}\Bigm]. 
\end{aligned}\end{equation}
'Sweep' requires $O(S)$ iterations and 'action cost' requires $O(A)$ iterations. Note that the reward penalty $\kappa(v)$ doesn't depend on state and action. It is calculated only once for value iteration for all states. The above value update has to be done for each states , so one full update requires
 \[ O\Bigm(S(\text{action cost}) (\text{sweep cost} \bigm) +  \text{reward cost}\Bigm)= O\Bigm(S^2A  +  \text{reward cost}\Bigm)\]
Since the value iteration is a contraction map, so to get $\epsilon$-close to the optimal value, it requires $O(\log(\frac{1}{\epsilon}))$ full value update, so the complexity is  
\[O\Bigm(\log(\frac{1}{\epsilon})\bigm(S^2A + \text{reward cost}\bigm)\Bigm).\]

\begin{enumerate}
    \item \textbf{Non-robust MDPs}: The cost of 'reward is zero as there is no regularizer to compute. The total complexity is   
     \[O\Bigm(\log(\frac{1}{\epsilon})\bigm(S^2A + 0 \bigm)\Bigm) = O\Bigm(\log(\frac{1}{\epsilon})S^2A\Bigm).\]
     
    \item  \textbf{$(\mathtt{sa})$-rectangular $L_1/L_2/L_\infty$ and $\mathtt{s}$-rectangular $L_\infty$ robust MDPs}: We need to calculate the reward penalty ($\kappa_1(v)/\kappa_2(v)/\kappa_\infty$)  that takes $O(S)$ iterations. As calculation of mean, variance and median, all are linear time compute.  Hence the complexity is
     \[O\Bigm(\log(\frac{1}{\epsilon})\bigm(S^2A + S\bigm)\Bigm) = O\Bigm(\log(\frac{1}{\epsilon})S^2A\Bigm).\]
     
  \end{enumerate}
   
\subsection{Exact Value iteration: Top $k$ response}
In this section, we discuss the time complexity of $\mathtt{s}$-rectangular $L_1/L_2$ robust MDPs as in algorithm \ref{alg:SLp}.  We need to calculate the reward penalty ($\kappa_\infty(v)/\kappa_2(v)$ in \eqref{alg:SLP:eq:kappa}) that  takes $O(S)$ iterations. Then for each state we do: sorting of Q-values in \eqref{alg:SLP:eq:Qsort}, value evaluation in \eqref{alg:SLP:eq:valEval}, update Q-value in \eqref{alg:SLP:eq:Qupdate} that takes $O(A\log(A)), O(A), O(SA)$ iterations respectively. 
    Hence the complexity is
    \[=\text{total iteration(reward cost \eqref{alg:SLP:eq:kappa} + S( sorting \eqref{alg:SLP:eq:Qsort} + value evaluation \eqref{alg:SLP:eq:valEval} +Q-value\eqref{alg:SLP:eq:Qupdate})}\]
    \[=\log(\frac{1}{\epsilon})(S + S(A\log(A) + A +SA)\]
     \[O\Bigm(\log(\frac{1}{\epsilon})\bigm(S^2A + SA\log(A)\bigm)\Bigm).\]

For general $p$, we need little caution as $k_p(v)$ can't be calculated exactly but approximately by binary search. And it is the subject of discussion for the next sections.

\subsection{Inexact Value Iteration: $(\mathtt{sa})$-rectangular $L_p$ robust MDPs ($\mathcal{U}^{\mathtt{sa}}_p$)}
In this section, we will study the time complexity for robust value iteration for $(\mathtt{sa})$-rectangular $L_p$ robust MDPs for general $p$. Recall, that value iteration takes best penalized action, that is easy to compute. But reward penalization depends on $p$-variance measure $\kappa_p(v)$, that we will estimate by $\hat{\kappa}_p(v)$ through binary search. We have inexact value iterations  as 
\[v_{n+1}(s) := \max_{a\in\mathcal{A}}[\alpha_{sa} - \gamma\beta_{sa}\hat{\kappa}_q(v_n) + R_0(s,a) + \gamma\sum_{s'}P_0(s'|s,a)v_n(s')]\]
where $\hat{\kappa}_q(v_n)$ is a $\epsilon_1$ approximation of $\kappa_q(v_n)$, that is $\lvert\hat{\kappa}_q(v_n) - \kappa_q(v_n)\rvert \leq \epsilon_1$. Then it is easy to see that we have bounded error in robust value iteration, that is 
\[\lVert v_{n+1} -\mathcal{T}^*_{\mathcal{U}^{\mathtt{sa}}_p}v_n\rVert_{\infty} \leq \gamma\beta_{max}\epsilon_1\]
where $\beta_{max} :=\max_{s,a}\beta_{s,a}$
\begin{proposition}\label{rs:appVI} Let $\mathcal{T}^*_{\mathcal{U}}$ be a $\gamma$ contraction map, and $v^*$ be its fixed point. And let $\{v_n,n\geq 0\}$ be approximate value iteration, that is 
\[ \lVert v_{n+1} -\mathcal{T}^*_{\mathcal{U}}v_n\rVert_\infty \leq\epsilon \]
then 
\[\lim_{n\to \infty}\lVert v_{n} -v^*\rVert_\infty \leq \frac{\epsilon}{1-\gamma}\]
moreover, it converges to the $\frac{\epsilon}{1-\gamma}$ radius ball linearly, that is
\[\lVert v_{n} -v^*\rVert_\infty-\frac{\epsilon}{1-\gamma} \leq c\gamma^n\]
where $c = \frac{1}{1-\gamma }\epsilon + \lVert v_0 -v^*\rVert_\infty$.
\end{proposition}
\begin{proof}
\begin{equation}\begin{aligned}
    \lVert v_{n+1} -v^*\rVert_\infty = & \lVert v_{n+1} -\mathcal{T}^*_{\mathcal{U}}v^*\rVert_\infty\\
    =&\lVert v_{n+1}- \mathcal{T}^*_{\mathcal{U}}v_n +\mathcal{T}^*_{\mathcal{U}}v_n -\mathcal{T}^*_{\mathcal{U}}v^*\rVert_\infty\\
    \leq&\lVert v_{n+1}- \mathcal{T}^*_{\mathcal{U}}v_n\rVert_\infty +\lVert \mathcal{T}^*_{\mathcal{U}}v_n -\mathcal{T}^*_{\mathcal{U}}v^*\rVert_\infty\\
    \leq&\lVert v_{n+1}- \mathcal{T}^*_{\mathcal{U}}v_n\rVert_\infty +\gamma\lVert v_n -v^*\rVert_\infty, \qquad \text{(contraction)}\\
    \leq&\epsilon +\gamma\lVert v_n -v^*\rVert_\infty, \qquad \text{(approximate value iteration)}\\
    \implies \lVert v_n -v^*\rVert_\infty = &\sum_{k=0}^{n-1}\gamma^k\epsilon + \gamma^n\lVert v_0 -v^*\rVert_\infty, \qquad \text{(unrolling above recursion)}\\
     = &\frac{1-\gamma^n}{1-\gamma }\epsilon + \gamma^n\lVert v_0 -v^*\rVert_\infty\\
     = &\gamma^n[\frac{1}{1-\gamma }\epsilon + \lVert v_0 -v^*\rVert_\infty] +\frac{\epsilon}{1-\gamma } 
\end{aligned}\end{equation}
Taking limit $n \to \infty$ both sides, we get 
\[\lim_{n\to \infty}\lVert v_{n} -v^*\rVert_\infty \leq \frac{\epsilon}{1-\gamma}.\]
\end{proof}
\begin{lemma}\label{rs:saLpCompl}For $\mathcal{U}^{\mathtt{sa}}_p$, the total iteration cost is $\log(\frac{1}{\epsilon})S^2A+ (\log(\frac{1}{\epsilon}))^2$ to get $\epsilon$ close to the optimal robust value function.
\end{lemma}
\begin{proof}
We calculate $\kappa_q(v)$ with $\epsilon_1 = \frac{(1-\gamma)\epsilon}{3}$ tolerance that takes $O(S\log(\frac{S}{\epsilon_1}))$ using binary search (see section \ref{app:BSkappa}). Now, we do approximate value iteration for $ n =\log(\frac{3\lVert v_0 -v^*\rVert_\infty}{\epsilon})$. Using the above lemma, we have
\begin{equation}\begin{aligned}
    \lVert v_{n} -v^*_{\mathcal{U}^{\mathtt{sa}}_p}\rVert_\infty= &\gamma^n[\frac{1}{1-\gamma }\epsilon_1 + \lVert v_0 -v^*_{\mathcal{U}^{\mathtt{sa}}_p}\rVert_\infty] +\frac{\epsilon_1}{1-\gamma } \\
    \leq &\gamma^n[\frac{\epsilon}{3} + \lVert v_0 -v^*_{\mathcal{U}^{\mathtt{sa}}_p}\rVert_\infty] +\frac{\epsilon}{3} \\
    \leq &\gamma^n\frac{\epsilon}{3} + \frac{\epsilon}{3} +\frac{\epsilon}{3} 
    \leq \epsilon.
\end{aligned}\end{equation}
In summary, we have action cost $O(A)$, reward cost $O(S\log(\frac{S}{\epsilon}))$, sweep cost $O(S)$ and total number of iterations $O(\log(\frac{1}{\epsilon}))$. So the complexity is 
\[\text{(number of iterations)\big(S(actions cost) (sweep cost) + reward cost\big)}\]
\[= \log(\frac{1}{\epsilon})\bigm(S^2A +   S\log(\frac{S}{\epsilon})\bigm) = \log(\frac{1}{\epsilon}) (S^2 A + S\log(\frac{1}{\epsilon}) +S\log(S) )\]
\[ = \log(\frac{1}{\epsilon})S^2A+ S(\log(\frac{1}{\epsilon}))^2 \]
\end{proof}

\subsection{Inexact Value Iteration: $\mathtt{s}$-rectangular $L_p$ robust MDPs ($\mathcal{U}^{\mathtt{s}}_p$)}
In this section, we study the time complexity for robust value iteration for \texttt{s}-rectangular $L_p$ robust MDPs for general $p$ ( algorithm \ref{alg:SALp}). Recall, that value iteration takes regularized actions and penalized reward. And reward penalization depends on $q$-variance measure $\kappa_q(v)$, that we will estimate by $\hat{\kappa}_q(v)$ through binary search, then again we will calculate $\mathcal{T}^*_{\mathcal{U}^{\mathtt{sa}}_p}$ by binary search with approximated $\kappa_q(v)$. Here, we have two error sources (\eqref{alg:SLP:eq:kappa}, \eqref{alg:SLP:eq:valEval}) as contrast to $(\mathtt{sa})$-rectangular cases, where there was only one error source from the estimation of $\kappa_q$.\\

First, we account for the error caused by the first source ($\kappa_q$). Here we do value iteration with approximated $q$-variance $\hat{\kappa}_q$, and exact action regularizer.
We have 
\[v_{n+1}(s) :=\lambda \quad \text{s.t. }\quad \alpha_s +\gamma\beta_{s}\hat{\kappa}_q(v) = (\sum_{Q(s,a)\geq \lambda}(Q(s,a) - \lambda)^{p})^{\frac{1}{p}}\]
where $Q(s,a) =R_0(s,a) + \gamma\sum_{s'}P_0(s'|s,a)v_n(s'),$ and $ \lvert\hat{\kappa}_q(v_n) - \kappa_q(v_n)\rvert \leq \epsilon_1$. Then from the next result (proposition \ref{rs:sLpkappaErr}), we get
\[\lVert v_{n+1} -\mathcal{T}^*_{\mathcal{U}^{\mathtt{sa}}_p}v_n\rVert_\infty \leq \gamma\beta_{max}\epsilon_1\]
where $\beta_{max} :=\max_{s,a}\beta_{s,a}$

\begin{proposition}\label{rs:sLpkappaErr} Let $\hat{\kappa}$ be an an $\epsilon$-approximation of $\kappa$, that is  $ \lvert \hat{\kappa} - \kappa\rvert \leq \epsilon$, and let $b\in\mathbb{R}^A$ be sorted component wise, that is, $b_1\geq, \cdots,\geq b_A$. Let $\lambda$ be the solution to the following equation with exact parameter $\kappa$,
\[\alpha +\gamma\beta\kappa = (\sum_{b_i\geq \lambda}|b_i - \lambda|^{p})^{\frac{1}{p}}\] and let $\hat{\lambda}$ be the solution of the following equation with approximated parameter $\hat{\kappa}$,
\[\alpha +\gamma\beta\hat{\kappa} = (\sum_{b_i\geq \hat{\lambda}}|b_i - \hat{\lambda}|^{p})^{\frac{1}{p}},\] 
then $\hat{\lambda}$ is an $O(\epsilon)$-approximation of $\lambda$, that is 
\[\lvert \lambda - \hat{\lambda}\rvert \leq \gamma\beta\epsilon.\]
\end{proposition}
\begin{proof}
Let the function $f:[b_A,b_1]\to\mathbb{R}$ be defined as  
\[f(x) := (\sum_{b_i\geq x}|b_i - x|^{p})^{\frac{1}{p}}.\]
We will show that derivative of $f$ is bounded, implying its inverse is bounded and hence Lipschitz, that will prove the claim. Let proceed 
\begin{equation}\begin{aligned}
    \frac{df(x)}{dx} &= -(\sum_{b_i\geq x}|b_i - x|^{p})^{\frac{1}{p} -1}\sum_{b_i\geq x}|b_i - x|^{p-1}\\
     &= - \frac{\sum_{b_i\geq x}|b_i - x|^{p-1}}{(\sum_{b_i\geq x}|b_i - x|^{p})^{\frac{p-1}{p}}}\\
     &= - \Bigm[\frac{(\sum_{b_i\geq x}|b_i - x|^{p-1})^{\frac{1}{p-1}}}{(\sum_{b_i\geq x}|b_i - x|^{p})^{\frac{1}{p}}}\Bigm]^{p-1}\\
     &\leq -1.
\end{aligned}
\end{equation}
The inequality follows from the following  relation between $L_p$ norm, 
\[\lVert x\rVert_a \geq \lVert x\rVert_b ,\qquad \forall 0\leq  a\leq b.   \]
It is easy to see that the function $f$ is strictly monotone in the range $b_A ,b_1]$, so its inverse is well defined in the same range. Then derivative of the inverse of the function $f$ is bounded as
\[ 0 \geq \frac{d}{dx}f^-(x)\geq -1.\]
Now, observe that $\lambda = f^-(\alpha +\gamma\beta\kappa)$ and $\hat{\lambda} = f^-(\alpha +\gamma\beta\hat{\kappa})$, then by Lipschitzcity, we have 
\[|\lambda - \hat{\lambda}| = |f^-(\alpha +\gamma\beta\kappa)- f^-(\alpha +\gamma\beta\hat{\kappa})| \leq \gamma\beta|-\kappa -\hat{\kappa})| \leq \gamma\beta\epsilon. \]
\end{proof}

\begin{lemma}For $\mathcal{U}^{\mathtt{s}}_p$, the total iteration cost is $O\Bigm(\log(\frac{1}{\epsilon})\bigm( S^2A+ SA\log(\frac{A}{\epsilon}) \bigm)\Bigm)$ to get $\epsilon$ close to the optimal robust value function.
\end{lemma}
\begin{proof}
We calculate $\kappa_q(v)$ in \eqref{alg:SLP:eq:kappa} with $\epsilon_1 = \frac{(1-\gamma)\epsilon}{6}$ tolerance that takes $O(S\log(\frac{S}{\epsilon_1}))$ iterations using binary search (see section \ref{app:BSkappa}). Then for every state, we sort the $Q$ values (as in \eqref{alg:SLP:eq:Qsort}) that costs $O(A\log(A))$ iterations. In each state, to update value, we do again binary search with approximate $\kappa_q(v)$ upto $\epsilon_2:=\frac{(1-\gamma)\epsilon}{6}$ tolerance, that takes $O(\log(\frac{1}{\epsilon_2}))$ search iterations and each iteration cost $O(A)$, altogether it costs $O(A\log(\frac{1}{\epsilon_2}))$ iterations. Sorting of actions and binary search adds upto $O(A\log(\frac{A}{\epsilon}))$ iterations (action cost).
So we have (doubly) approximated value iteration  as following, 
\begin{equation}
    \lvert v_{n+1}(s)  -\hat{\lambda}\rvert \leq \epsilon_1
\end{equation}
where
\[(\alpha_s +\gamma\beta_{s}\hat{\kappa}_q(v_n))^p = \sum_{Q_n(s,a)\geq \hat{\lambda}}(Q_n(s,a) - \hat{\lambda})^{p}\]
and  \[Q_n(s,a) =R_0(s,a) + \gamma\sum_{s'}P_0(s'|s,a)v_n(s'), \qquad \lvert \hat{\kappa}_q(v_n) -\kappa_q(v_n)\rvert \leq \epsilon_1.\] 
And we do this approximate value iteration for $ n =\log(\frac{3\lVert v_0 -v^*\rVert_\infty}{\epsilon})$. Now, we do error analysis. By accumulating error, we have
\begin{equation}\begin{aligned}
    \lvert v_{n+1}(s) -(\mathcal{T}^*_{\mathcal{U}^{\mathtt{s}}_p}v_n)(s)\rvert \leq& \lvert v_{n+1}(s) -\hat{\lambda}\rvert + \lvert \hat{\lambda}-(\mathcal{T}^*_{\mathcal{U}^{\mathtt{s}}_p}v_n)(s)\rvert\\
    \leq& \epsilon_1 +\lvert\hat{\lambda}-(\mathcal{T}^*_{\mathcal{U}^{\mathtt{s}}_p}v_n)(s)\rvert , \qquad \text{(by definition)}\\
    \leq &\epsilon_1 + \gamma\beta_{\max}\epsilon_1, \qquad \text{(from proposition \ref{rs:sLpkappaErr})}\\
    \leq& 2\epsilon_1.
\end{aligned}\end{equation}

where $\beta_{max}:=\max_{s}\beta_s, \gamma \leq 1$.

Now, we do approximate value iteration, and from proposition \ref{rs:appVI}, we get 
\begin{equation}\begin{aligned}
    \lVert v_{n} -v^*_{\mathcal{U}^s_p}\rVert
    \leq& \frac{2\epsilon_1}{1-\gamma} + \gamma^n[\frac{1}{1-\gamma }2\epsilon_1 + \lVert v_0 -v^*_{\mathcal{U}^s_p} \rVert_\infty]
\end{aligned}\end{equation}
Now, putting the value of $n$, we have
\begin{equation}\begin{aligned}
    \lVert v_{n} -v^*_{\mathcal{U}^s_p}\rVert_\infty= &\gamma^n[\frac{2\epsilon_1}{1-\gamma } + \lVert v_0 -v^*_{\mathcal{U}^s_p}\rVert_\infty] +\frac{2\epsilon_1}{1-\gamma } \\
    \leq &\gamma^n[\frac{\epsilon}{3} + \lVert v_0 -v^*_{\mathcal{U}^s_p}\rVert_\infty] +\frac{\epsilon}{3} \\
    \leq &\gamma^n\frac{\epsilon}{3} + \frac{\epsilon}{3} +\frac{\epsilon}{3} 
    \leq \epsilon.
\end{aligned}\end{equation}
To summarize, we do $O(\log(\frac{1}{\epsilon}))$ full value iterations. Cost of evaluating reward penalty is $O(S\log(\frac{S}{\epsilon}))$. For each state: evaluation of $Q$-value from value function requires $O(SA)$ iterations, sorting the actions according $Q$-values requires $O(A\log(A))$ iterations, and binary search for evaluation of value requires $O(A\log(1/\epsilon)$. So the complexity is 
\[O(\text{(total iterations)(reward cost + S(Q-value + sorting + binary search for value )}))\]
\[= O\Bigm(\log(\frac{1}{\epsilon})\bigm(S\log(\frac{S}{\epsilon}) + S(SA +A\log(A)+ A\log(\frac{1}{\epsilon})) \bigm)\Bigm)\]
\[= O\Bigm(\log(\frac{1}{\epsilon})\bigm(S\log(\frac{1}{\epsilon}) + 
S\log(S) + S^2A+SA\log(A) + SA\log(\frac{1}{\epsilon}) \bigm)\Bigm)\]
\[= O\Bigm(\log(\frac{1}{\epsilon})\bigm( S^2A+SA\log(A) + SA\log(\frac{1}{\epsilon}) \bigm)\Bigm)\]

\[= O\Bigm(\log(\frac{1}{\epsilon})\bigm( S^2A+ SA\log(\frac{A}{\epsilon}) \bigm)\Bigm)\]
\end{proof}

\end{document}